\documentclass[]{article}

\PassOptionsToPackage{numbers,sort&compress}{natbib}
\usepackage[preprint]{neurips_2026} %, final

\usepackage[T1]{fontenc}
\usepackage{microtype}
\usepackage[utf8]{inputenc} % allow utf-8 input

\usepackage{amsmath,amsthm,amsfonts,amssymb}

\DeclareMathAlphabet{\mathmybb}{U}{bbold}{m}{n}

\usepackage{hyperref}
\usepackage[capitalise]{cleveref}
\usepackage{enumitem}
\usepackage{url}

\usepackage{xcolor}
\usepackage{subcaption}
\usepackage{booktabs}

\usepackage{tikz}
\usetikzlibrary{
    calc,
    angles,
    plotmarks,
    shapes,
    fit,
    decorations.markings
}

\usepackage{tikz-cd}

\usepackage{pgfplots}
\usepackage{pgfplotstable}
\pgfplotsset{compat=newest}
\usepgfplotslibrary{statistics,groupplots}
\pgfplotsset{
    discard if not/.style 2 args={
        x filter/.code={
            \edef\tempa{\thisrow{#1}}
            \edef\tempb{#2}
            \ifx\tempa\tempb
            \else
                
            \fi
        }
    }
}

\pgfdeclarelayer{back}
\pgfdeclarelayer{front}
\pgfsetlayers{back,main,front}

\usepackage{algorithm}
\usepackage{algpseudocode}

\usepackage{mdframed}

\newcommand{\bb}{\boldsymbol{b}}
\newcommand{\bc}{\boldsymbol{c}}
\newcommand{\be}{\boldsymbol{e}}
\newcommand{\bm}{\boldsymbol{m}}
\newcommand{\bu}{\boldsymbol{u}}
\newcommand{\bv}{\boldsymbol{v}}
\newcommand{\bx}{\boldsymbol{x}}
\newcommand{\by}{\boldsymbol{y}}
\newcommand{\bz}{\boldsymbol{z}}

\newcommand{\bA}{\boldsymbol{A}}
\newcommand{\bI}{\boldsymbol{I}}
\newcommand{\bJ}{\boldsymbol{J}}
\newcommand{\bK}{\boldsymbol{K}}
\newcommand{\bM}{\boldsymbol{M}}
\newcommand{\bP}{\boldsymbol{P}}
\newcommand{\bQ}{\boldsymbol{Q}}
\newcommand{\bW}{\boldsymbol{W}}
\newcommand{\bX}{\boldsymbol{X}}

\newcommand{\bmu}{\boldsymbol{\mu}}
\newcommand{\balpha}{\boldsymbol{\alpha}}
\newcommand{\beps}{\boldsymbol{\epsilon}}

\newcommand{\R}{\mathbb{R}}

\makeatletter
\newsavebox{\@brx}
\newcommand{\llangle}[1][]{%
  \savebox{\@brx}{\(\m@th{#1\langle}\)}%
  \mathopen{\copy\@brx\kern-0.5\wd\@brx\usebox{\@brx}}}
\newcommand{\rrangle}[1][]{%
  \savebox{\@brx}{\(\m@th{#1\rangle}\)}%
  \mathclose{\copy\@brx\kern-0.5\wd\@brx\usebox{\@brx}}}
\makeatother

\newtheorem{theorem}{Theorem}
\newtheorem{lemma}{Lemma}
\newtheorem{corollary}{Corollary}
\newtheorem{proposition}{Proposition}

\crefname{conjecture}{Conjecture}{Conjectures}
\crefname{claim}{Claim}{Claims}

\theoremstyle{definition}
\newtheorem{definition}{Definition}

\crefname{assumption}{Assumption}{Assumptions}
\definecolor{c0}{HTML}{4477AA}
\definecolor{c1}{HTML}{EE6677}
\definecolor{c10}{HTML}{228833}
\definecolor{c100}{HTML}{CCBB44}
\definecolor{c1000}{HTML}{AA3377}
 
\pgfplotstableread[col sep=comma]{data/summary_metrics.csv}\summarydata
\pgfplotstableread[col sep=comma]{data/history_gamma_0.csv}\histzero
\pgfplotstableread[col sep=comma]{data/history_gamma_1.csv}\histone
\pgfplotstableread[col sep=comma]{data/history_gamma_10.csv}\histten
\pgfplotstableread[col sep=comma]{data/history_gamma_100.csv}\hstone
\pgfplotstableread[col sep=comma]{data/history_gamma_1000.csv}\histthousand

\definecolor{vir1}{HTML}{440154}
\definecolor{vir2}{HTML}{414487}
\definecolor{vir3}{HTML}{2a788e}
\definecolor{vir4}{HTML}{22a884}
\definecolor{vir5}{HTML}{7ad151}
\definecolor{vir6}{HTML}{fde725}

\title{The Geometric Structure of Models\\Learning Sparse Data}
\author{
  Thomas Walker\thanks{Correspondence: \texttt{thomas.walker@rice.edu}} \\
  Rice University\\
  \And
  T. Mitchell Roddenberry \\
  Rice University\\
  \AND
  Ahmed Imtiaz Humayun \\
  Rice University\thanks{Now at Google Research} \\
  \And
  Randall Balestriero \\
  Brown University \\
  \And
  Richard Baraniuk \\
  Rice University
}

\begin{document}

\maketitle

\begin{abstract}
    The manifold hypothesis (MH) is often used to explain how machine learning can overcome the {\em curse of dimensionality}.
    However, the MH is only applicable in regimes where the training data provides a sufficiently dense sample of the underlying low-dimensional data manifold, or where such a low-dimensional manifold is conceivably present.
    We describe the regimes where the MH is not applicable as \textbf{sparse}.
    In this paper, we demonstrate that models succeed in the sparse regime by exploiting a highly structured local geometry, a property we formalize as \textbf{normal alignment}. 
    We prove that normal-aligned classifiers—whose input-output Jacobians are rank-one and align perfectly with the training data—minimize the training objective under norm constraints and achieve maximal local robustness under a non-zero Jacobian constraint. 
    For continuous piecewise-affine deep networks, normal alignment manifests geometrically as \textit{centroid alignment} within the network's induced power diagram partition and results from the feature-learning regime. 
    Motivated by these theoretical insights, we introduce \textbf{GrokAlign}, a regularization strategy that actively induces normal alignment.
    We demonstrate that GrokAlign significantly accelerates the training dynamics of deep networks relevant to the grokking phenomenon.
    Furthermore, we apply the principle of normal alignment to Recursive Feature Machines (RFMs) to introduce \textbf{Recursive Feature Alignment Machines (RFAMs)}.
    We show that RFAMs exhibit greater adversarial robustness compared to RFMs when trained on tabular data.
\end{abstract}

\section{Introduction}

The manifold hypothesis (MH) states that high-dimensional natural data has the majority of its structure in a low-dimensional subspace~\cite{tenenbaumGlobalGeometricFramework2000,roweisNonlinearDimensionalityReduction2000}.
Machine learning practitioners use this as a post-hoc explanation of how models can overcome the {\em curse of dimensionality} and extract patterns from high-dimensional datasets.

The MH assumes a continuous low-dimensional structure; thus, to be applicable in practice, it is necessary that the training data is a sufficiently dense sample.
Moreover, the assumed low-dimensional structure ought to be sufficiently regular with respect to the natural inductive biases of machine learning models.
Indeed, many common machine learning practices can be viewed as ensuring these conditions hold.
For example, data augmentation ensures a dense sampling of the manifold~\cite{zhangMixupEmpiricalRisk2018,yunCutMixRegularizationStrategy2019,takahashiDataAugmentationUsing2020,zhongRandomErasingData2020,cubukAutoaugmentLearningAugmentation2019a,hendrycksAugMixSimpleMethod2020}, and model architectures are designed to complement the known structures of the data~\cite{lecunBackpropagationAppliedHandwritten1989,vaswaniAttentionAllYou2017}. 

Yet, machine learning models routinely succeed in regimes where the manifold hypothesis breaks down. 
This occurs primarily under two conditions: first, in data-scarce settings where standard data augmentation cannot be readily applied, making a dense sample of the input space impossible to obtain. 
Second, in inherently discrete domains—such as the algorithmic task of modular addition—the concept of a continuous underlying data manifold remains entirely inapplicable, even under infinite data assumptions.

In this paper, we show that the success of machine learning models in these {\em sparse} settings can be similarly attributed to the exploitation of low-dimensional structures.

Since the phenomenon of grokking -- performance on the train set saturating well before performance on a test set saturates~\cite{powerGrokkingGeneralizationOverfitting2022} -- is a canonical example of a sparse setting, we use these insights to introduce the \textbf{GrokAlign} strategy for accelerating grokking training dynamics.
Similarly, as tabular data can be described as sparse, we introduce \textbf{Recursive Feature Alignment Machines (RFAMs)} to improve the robustness of Recursive Feature Machines (RFMs)~\cite{radhakrishnanMechanismFeatureLearning2024} when trained on tabular data.

\section{Sparse Datasets and Normal Aligned Classifiers}\label{sec:sparsity_and_alignment}

Let $f:\R^d\to\R^C$ be a classifier that predicts the class of an input $\bx\in\R^d$ as $\mathrm{argmax}(f(\bx))$.
Let $\bJ_{\bx}\in\R^{C\times d}$ denote the input-output Jacobian of $f$ at $\bx$, with the $c^\text{th}$ row denoted as $\bJ^{(c)}_{\bx}$, and $\boldsymbol{\nu}_{\bx}^{(k)}\in\R^d$ denoting its $k^{\text{th}}$ top right singular vector with corresponding singular value $\sigma_{\bx}^{(k)}\in\R$.
The output of the classifier can be decomposed as $f(\bx)=\bJ_{\bx}\bx+\bb_{\bx}$, where $\bb_{\bx}\in\R^C$ is the offset of $f$ at $\bx$.
Classifiers are trained on a data set $\mathcal{D}=\left\{\left(\bx_i,y_i\right)\right\}_{i=1}^n$ -- where $\bx_i\in\R^d$ and $y_i\in\R$ is its corresponding class -- under some loss function $\mathcal{L}=\frac{1}{n}\sum_{i=1}^n\ell\left(f\left(\bx_i\right),y_i\right)=:\frac{1}{n}\sum_{i=1}^n\ell_i$.

\begin{definition}\label{def:jacobian_aligned}
    A classifier $f$ is \textbf{normal-aligned} to a dataset $\mathcal{D}$ if for every $i\in\{1,\dots,n\}$ there exists $\bc_{i}\in\R^C$ such that $\bJ_{\bx_i}=\bc_i\bx_i^\top$.
\end{definition}

That is, a model $f$ is normal-aligned to $\mathcal{D}$ if, at each training point $\bx_i$ in $\mathcal{D}$, it only varies along the direction of $\bx_i$.
In particular, the Jacobian of a model at $\bx_i$ is {\em rank-one} with rows that are scalar multiples of $\bx_i$.

Recalling how the output of the classifier is reconstructed as $f(\bx_i)=\bJ_{\bx_i}\bx_i+\bb_{\bx_i}$, normal alignment resonates strongly with the concept of a {\em matched filter bank} from classical detection theory (radar and sonar, particularly)~\cite{balestrieroSplineTheoryDeep2018,balestrieroMadMaxAffine2020}.
A matched filter bank classifies a signal by selecting the template that maximizes the inner product between the signal and the template.
When the input signal $\bx_i$, this is done optimally by setting the template equal to $\bx_i$ (by the Cauchy-Schwarz inequality).
Thus, a normal-aligned classifier precisely implements the program of an optimal match filter on the training data.

\subsection{Sparse Datasets}

The structure of a trained model can reveal information about the data it was trained on.
The geometry of a normal-aligned classifier is locally one-dimensional, since variations only occur along the directions of the training data.
Meaning, from a normal-aligned model's perspective, there is no underlying structure connecting the dataset.
In other words, the dataset is {\em sparse}.

\begin{definition}
    A dataset $\mathcal{D}$ is \textbf{sparse} with respect to a class of classifiers $\mathcal{H}$ if there exists a subclass of classifiers $\mathcal{H}^{\prime}$ normal aligned to $\mathcal{D}$ such that for every $i=\left\{1,\dots,n\right\}$ the value of $\ell_i$ can be changed independently of $\ell_j$ for $j\neq i$.
\end{definition}

Normal-aligned classifiers are realizable for parameterized models such as deep networks.
In \Cref{sec:constructing_normal_aligned_deep_networks}, we demonstrate how a single hidden-layer deep network can be constructed to be normal aligned to a data set $\mathcal{D}$.
The construction demonstrates that as the size of $\mathcal{D}$ increases (i.e., it becomes more dense), a normal-aligned classifier becomes increasingly irregular from the perspective of weight norm.

\subsection{Properties of Normal Aligned Classifiers}\label{sec:properties}

\paragraph{Optimizing the training objective.}

Although normal alignment appears to be a restrictive property, it turns out to be optimal in the sparse regime.

\begin{theorem}\label{thm:optimal_jacobian}
    Let $\ell$ be a convex, non-negative, and differentiable and suppose $\mathcal{D}$ is sparse.
    Then under the constraint that $\left\Vert\bJ_{\bx_i}\right\Vert_F^2+\left\Vert\bb_{\bx_i}\right\Vert_2^2\leq\alpha$ for $i\in\{1,\dots,n\}$, the classifier which minimizes $\mathcal{L}$ is such that $\bJ_{\bx_i}=\bc_i\bx_i^{\top}$ for $i\in\{1,\dots,n\}$, and where $\bc_i$ only depends on $\bx_i$ through its norm and $y_i$.
\end{theorem}

\textit{Proof.} See \Cref{proof:optimal_jacobian}. \qed

\Cref{thm:optimal_jacobian} says that at the optimum of the training objective -- under a constraint on the norm of its Jacobian and offset terms at the training data -- the classifier is normal-aligned.
The norm constraint of \Cref{thm:optimal_jacobian} is analogous to the one imposed by weight-decay; it is also known to prevent ``gradient explosion'' \cite{bengioLearningLongtermDependencies1994,pascanuDifficultyTrainingRecurrent2013}.

We explore the nature of the alignment for different loss functions in \Cref{sec:alignment_loss_functions}.

\paragraph{Exhibiting robustness.}

In practice, it is desirable to have a robust classifier that correctly classifies the training data.
That is, a classifier that is invariant to small perturbations of the input.
To study this formally, fix an input point $\bx$ with true class $y$.
Let $\mathcal{G}_{\bx,\lambda}$ be the set of linear models $g(\bz)=\bJ\bz+\bb$ such that $\arg\max\left(g(\bx)\right)=y$ and $\frac{\left\vert[\bb]_y-[\bb]_c\right\vert}{\left\Vert\bJ^{(y)}-\bJ^{(c)}\right\Vert_2}\leq\lambda$ for $c\neq y$.

While this second condition prevents the trivial solution of a constant function with zero Jacobian, it naturally arises from standard regularized training. 
Specifically, weight decay strictly bounds the magnitude of the local offset $\bb_x$. 
Consequently, to obtain a suitable classification margin, the network is forced to maintain non-zero Jacobians and satisfy this condition. 
Forcing Jacobians to zero geometrically collapses the model to a bounded constant function within each activation polytope, destroying its expressivity and leading to severe performance degradation. 
We outline this argument in more detail in \Cref{sec:zero_jacobians} as well as verify it empirically.

Let the local robustness of $g\in\mathcal{G}_{\bx,\lambda}$ be given by
\begin{equation*}
    \rho\left(\bx;g\right)=\inf\left(\left\{\left\Vert\boldsymbol{\epsilon}\right\Vert_2:\arg\max(g(\bx+\boldsymbol{\epsilon}))\neq y\right\}\right),
\end{equation*}
such that the maximum local robustness is given by $\rho(\bx)=\sup_{g\in\mathcal{G}_{\bx,\lambda}}\left(\rho\left(\bx;g\right)\right)$.
For an arbitrary classifier $f$, we can speak of its local robustness at $\bx$ as $\rho\left(\bx;\tilde{f}\right)$ where $\tilde{f}\in\mathcal{G}_{\bx,\lambda}$ is given by $\tilde{f}(\bz)=\bJ_{\bx}\bz+\bb_{\bx}$.

\begin{theorem}\label{thm:optimally_robust}
    The normal aligned classifier $f$ of \cref{thm:optimal_jacobian} achieves maximal local robustness on $\mathcal{D}$.
\end{theorem}

\textit{Proof.} See \Cref{proof:optimally_robust}. \qed

To understand whether a normal-aligned classifier is achieved in practice, we consider the normal alignment and the effective rank of a model to the set $\mathcal{D}$ to be
\begin{equation*}
    \mathrm{na}(f;\mathcal{D})=\frac{1}{\vert\mathcal{D}\vert}\sum_{\bx\in\mathcal{D}}\left\vert\frac{\left\langle\boldsymbol{\nu}_{\bx}^{(1)},\bx\right\rangle}{\left\Vert\boldsymbol{\nu}_{\bx}^{(1)}\right\Vert_2\left\Vert\bx\right\Vert_2}\right\vert,\quad\text{and}\quad\mathrm{er}(f;\mathcal{D}):=\frac{1}{\vert\mathcal{D}\vert}\sum_{\bx\in\mathcal{D}}\frac{\sum_{r=1}^C\left(\sigma_{\bx}^{(r)}\right)^2}{\left(\sigma^{(1)}_{\bx}\right)^2},
\end{equation*}
respectively.

\subsection{The Emergence of Sparsity}

Sparsity depends on both the dataset and the model.
On the one hand, if a dataset has a low sample density, then it is more likely to be sparse.
In the image domain, techniques such as data augmentation~\cite{zhangMixupEmpiricalRisk2018,yunCutMixRegularizationStrategy2019,takahashiDataAugmentationUsing2020,zhongRandomErasingData2020,cubukAutoaugmentLearningAugmentation2019a,hendrycksAugMixSimpleMethod2020} can increase the sample density of datasets, moving them away from the sparse regime.
In the left panel of \Cref{fig:data_sparsity}, we train a fully connected deep network on a subset of MNIST~\cite{lecunGradientBasedLearningApplied1998} with different amounts of data augmentation.
As the amount of data augmentation is increased, the amount of normal alignment decreases.

On the other hand, if a model's capacity is increased, its ability to normal-align to a given dataset increases.
Similarly, different inductive biases of a model architecture may make it more or less capable of exhibiting normal alignment.
With the right panel of \Cref{fig:data_sparsity}, we observe that as model capacity increases, a greater amount of normal alignment is exhibited for a fixed amount of CIFAR10~\cite{krizhevskyLearningMultipleLayers2009} training data.

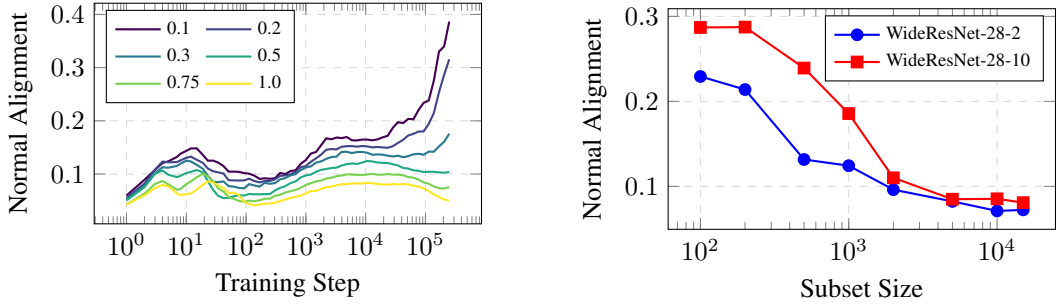
\begin{figure}[ht]
    \centering
    \begin{tikzpicture}
        \begin{axis}[
            xlabel={Training Step},
            ylabel={Normal Alignment},
            width=0.48\textwidth,
            height=4.5cm,
            grid=major,
            grid style={dashed, gray!30},
            xmode=log,
            legend pos=north west,
            legend columns=2,
            legend style={font=\scriptsize, cells={anchor=west}},
        ]
        \addplot[
            color=vir1,
            thick
        ] table [
            col sep=comma, 
            x=step, 
            y=intensity_0.1
        ] {data/mnist_alignment.csv};
        \addlegendentry{0.1}
        \addplot[
            color=vir2,
            thick
        ] table [
            col sep=comma, 
            x=step, 
            y=intensity_0.2
        ] {data/mnist_alignment.csv};
        \addlegendentry{0.2}
        \addplot[
            color=vir3,
            thick
        ] table [
            col sep=comma, 
            x=step, 
            y=intensity_0.3
        ] {data/mnist_alignment.csv};
        \addlegendentry{0.3}
        \addplot[
            color=vir4,
            thick
        ] table [
            col sep=comma, 
            x=step, 
            y=intensity_0.5
        ] {data/mnist_alignment.csv};
        \addlegendentry{0.5}
        \addplot[
            color=vir5,
            thick
        ] table [
            col sep=comma, 
            x=step, 
            y=intensity_0.75
        ] {data/mnist_alignment.csv};
        \addlegendentry{0.75}
        \addplot[
            color=vir6,
            thick
        ] table [
            col sep=comma, 
            x=step, 
            y=intensity_1
        ] {data/mnist_alignment.csv};
        \addlegendentry{1.0}
        \end{axis}
    \end{tikzpicture}
    \hfill
    \begin{tikzpicture}
        \begin{axis}[
            xlabel={Subset Size},
            ylabel={Normal Alignment},
            width=0.48\textwidth,
            height=4.5cm,
            grid=major,
            grid style={dashed, gray!30},
            xmode=log,
            legend pos=north east,
            legend style={font=\scriptsize, cells={anchor=west}},
        ]
        \addplot[
            color=blue,
            mark=*,
            mark options={fill=blue},
            thick
        ] table [
            col sep=comma, 
            x=subset_size, 
            y=wrn28_2
        ] {data/c10_alignment.csv};
        \addlegendentry{WideResNet-28-2}
        \addplot[
            color=red,
            mark=square*,
            mark options={fill=red},
            thick
        ] table [
            col sep=comma, 
            x=subset_size, 
            y=wrn28_10
        ] {data/c10_alignment.csv};
        \addlegendentry{WideResNet-28-10}
        
        \end{axis}
    \end{tikzpicture}
    \caption{
    \textbf{Dataset sparsity is a property of the dataset and model.}
    In the left panel, we monitor the normal alignment during deep network training on a subset of MNIST across varying intensities of data augmentation.
    In the right panel, we train wide residual deep network architectures~\cite{zagoruyko2017wide} robustly on different subset sizes of CIFAR10.
    At the end of training, we monitor the models' normal alignment.
    For more experimental details, see \Cref{exp:data_sparsity}.
    }
    \label{fig:data_sparsity}
\end{figure}

\section{Normal Alignment for Different Model Architectures}\label{sec:alignment_model}

Thus far, the theory of normal alignment has been agnostic of the model's architecture.
Understanding the interplay between normal alignment and a model architecture can elucidate the {\em intrinsic} utility of different architectures, and inform strategies to improve their implementation.

In this section, we consider normal alignment in deep networks and recursive feature machines~\cite{radhakrishnanMechanismFeatureLearning2024}.
We also consider the simpler example of a Gaussian kernel logistic regression model in \Cref{sec:gaussian_kernel_logistic_regression_models}.

\subsection{Deep Networks}

To develop the theory for deep networks, we focus on continuous piecewise-affine (CPA) deep networks.
Neoclassical classifiers like ReLU DNs are exactly CPA, while modern classifiers like transformers are approximately CPA~\cite{balestrieroSplineTheoryDeep2018,balestrieroMadMaxAffine2020}.
Since Jacobians are computable for arbitrary differentiable functions, conclusions in the CPA setting will transfer to the general setting.

\subsubsection{Continuous Piecewise Affine Deep Networks}

Any deep network built using affine transformations (e.g., convolution, matrix multiplication) and piecewise linear operations (e.g., ReLU activation, max pooling) is a CPA spline~\cite{balestrieroSplineTheoryDeep2018,balestrieroMadMaxAffine2020}.
CPA deep networks have two tightly interconnected features: 
(i)~An irregular tessellation (aka tiling or partition) $\Omega$ of the $d$-dimensional input space into convex polytopes~\cite{balestrieroGeometryDeepNetworks2019}.
(ii)~A collection of affine mappings (one for each polytope) arranged such that the overall input-to-output mapping is continuous. 
These polytopes and affine mappings combine together into the representation
\begin{equation}
f(\bx)=\sum_{\omega\in\Omega}(\bA_{\omega}\bx+\bb_{\omega})\mathbb{I}_{\{\bx\in\omega\}},
\label{eq:spline1}
\end{equation}
where the ``slopes'' matrix $\bA_{\omega}$ and ``intercept'' vector $\bb_{\omega}$ define the affine transformation mapping all inputs from tile $\omega\in\Omega$ to the output.
While it is not explicit in \eqref{eq:spline1}, $\Omega$, $\bA_{\omega}$, and $\bb_{\omega}$ conspire such that the overall mapping $f$ is continuous.
% The affine mapping parameters $\bA_{\omega},\bb_{\omega}$ can be computed directly from the deep network parameters (weights and biases)~\cite{balestrieroSplineTheoryDeep2018,balestrieroMadMaxAffine2020}.
Clearly, we see that the Jacobian and offset of a CPA classifier at input point $\bx$ are given simply by $\bJ_{\bx}=\bA_{\omega}$ and $\bb_{\bx}=\bb_{\omega}$,
where $\bx\in\omega$.
Of course, the Jacobian does not exist at the tile boundaries, but this set has measure zero in the input space.

The input-space tiling $\Omega$ is implicitly defined in terms of the weights and biases of a CPA deep network.
In brief, for a ReLU DN (see~\cite{balestrieroSplineTheoryDeep2018,balestrieroMadMaxAffine2020} for more details), the inference computation at each neuron in a deep network layer involves the inner product of the layer's input with the corresponding row of the layer's weight matrix.
Combined with the additive bias term, this computation defines a hyperplane that divides the layer's input space into two half-spaces.
The tiles are formed by the combinatorial intersections of these half-spaces.
Chaining layers together leads to a subdivision process that creates an increasingly fine tiling~\cite{humayunSplineCamExactVisualization2023}.

To consider this more precisely, we introduce the following notation.
For a CPA deep network $f=\left(f^{(L)}\circ\dots\circ f^{(1)}\right)$, each layer $f^{d^{(l)}}:\R^{(l-1)}\to\R^{d^{(l)}}$ and sub-component $f^{(1\leftarrow l)}=\left(f^{(l)}\circ\dots\circ f^{(1)}\right):\R^d\to\R^{d^{(l)}}$ is also a CPA deep network. 
Let $\bA^{(l)}_{\omega^{(l)}}$, $\bb^{(l)}_{\omega^{(l)}}$, $\omega^{(l)}$, $\Omega^{(l)}$ and $\bA^{(1\leftarrow l)}_{\omega^{(1\leftarrow l)}}$, $\bb^{(1\leftarrow l)}_{\omega^{(1\leftarrow l)}}$, $\omega^{(1\leftarrow l)}$, $\Omega^{(1\leftarrow l)}$ be analogous notation as introduced for CPA deep networks.

\subsubsection{The Theory of Centroid Alignment}

The $l^{\text{th}}$ layer of a CPA deep network partitions its input space as a power diagram $\Omega^{(l)}\subseteq\R^{(l-1)}$, making the partition of the input space of a deep network a power diagram subdivision~\cite{balestrieroGeometryDeepNetworks2019}.
Power diagrams are closely related to Voronoi diagrams but employ a different defining distance~\cite{rogersPackingCovering1964}.

\begin{definition}
Given a collection of $Q$ {\bf centroid-radius} pairs 
$\left\{\left(\bmu_q,r_q\right)\right\}_{q=1}^{Q}\subseteq\R^d\times\R$, 
a {\bf power diagram} partitions $\R^d$ into $Q$ disjoint tiles $\Omega=\{\omega_1,\dots,\omega_Q\}$ such that 
$\cup_{q=1}^Q\omega_q=\R^d$,
with each tile given by
\begin{equation}\label{eq:power_diagram}
    \omega_{q}=\left\{\bx\in\R^d:q=\arg\min_{q'\in\{1,\dots,Q\}}\left(\left\Vert\bx-\bmu_{q'}\right\Vert_2^2-r_{q'}\right)\right\}.
\end{equation}
\end{definition}

The distance minimized in \eqref{eq:power_diagram} is called the Laguerre distance~\cite{imaiVoronoiDiagramLaguerre1985}.

The power diagram subdivision induced by the deep network is then constructed recursively.
The first layer of the DN partitions the input space $\R^d$ as $\Omega^{(1)}$.
The second layer then partitions the projections of the tiles $\omega^{(1)}\subseteq\R^d$ in $\R^{d^{(1)}}$ induced by $f^{(1)}$.
These are then pulled back to $\R^d$ to yield a finer partition of the input space $\Omega^{(1\leftarrow 2)}$ which is a power diagram subdivision.
Eventually, the power diagram subdivision $\Omega=\Omega^{(1\leftarrow L)}$ of the DN is constructed.
This process is analogous to hierarchical $k$-means~\cite{nister2006scalable}.
For more details, consult~\cite{balestrieroGeometryDeepNetworks2019}, which shows that one can similarly obtain descriptors for a region $\omega_q^{(1\leftarrow l)}$ of a power diagram division $\Omega^{(1\leftarrow l)}$, namely $\bmu_q^{(1\leftarrow l)}$ and $r_q^{(1\leftarrow l)}$.
For simplicity, we will still refer to $\bmu_q^{(1\leftarrow l)}$ and $r_q^{(1\leftarrow l)}$ as centroids and radii; however, it is important to note that they do not reconstruct the partition $\Omega^{(1\leftarrow l)}$ through \eqref{eq:power_diagram}.

While each tile $\omega$ in the partition is defined implicitly through a combinatorial intersection of half-spaces, its centroid and radius are defined explicitly.
In a Voronoi diagram, centroids can be interpreted as elements of the input space $\R^d$. 
In contrast to Voronoi diagrams, in a power diagram (subdivision), a centroid is likely to lie outside its polytope.

Let $\varphi(\bx)=q$ where $\bx\in\omega_q$, and define the all-ones vector by $\mathbf{1}$.

\begin{proposition}\label{prop:pd_parameters}
    For a CPA deep network $f$, we have $\bmu_{\varphi(\bx)}=\bJ_{\bx}^\top\mathbf{1}$ and $r_{\varphi(\bx)}=\left\Vert\bmu_{\varphi(\bx)}\right\Vert_2^2+2\bb_{\bx}^{\top}\mathbf{1}$. 
\end{proposition}

\textit{Proof.} See \Cref{proof:pd_parameters}. \qed

In words, {\em the centroid of a polytope is the row-sum of its Jacobian}.
\Cref{prop:pd_parameters} enables us to access the centroid and radius of polytopes through an efficient Jacobian-vector product computation for arbitrary deep networks, including transformers~\cite{vaswaniAttentionAllYou2017}.

The connection established between a deep network's power diagram parameters and its Jacobian enables us to examine the implications of normal alignment in deep networks.
% Again, for simplicity and without loss of generality, we consider bias-free deep networks and focus on the centroids of a deep network.

\begin{definition}
    A deep network $f$ is \textbf{centroid-aligned} to $\mathcal{D}$ if for every $i\in\{1,\dots,n\}$ there exists $c_i\in\R$ such that $\bmu_{\varphi(\bx_i)}=c\bx_i$.
\end{definition}

From Definition~\ref{def:jacobian_aligned} and \Cref{prop:pd_parameters}, it is clear that centroid-alignment is a weaker property than normal-alignment.

\begin{corollary}\label{cor:jacobian_aligned_implies_centroid_aligned}
    A deep network normal-aligned to $\mathcal{D}$ is centroid-aligned to $\mathcal{D}$. 
\end{corollary}

The proof follows by direct calculation of the centroid of an aligned Jacobian:  $\bmu_{\varphi(\bx)}=\bc\bx^\top\mathbf{1}=\bx\bc^\top\mathbf{1}=c\bx$, where $c=\bc^\top\mathbf{1}$.

In \Cref{fig:ca_modadd}, we verify that a one-layer transformer deep network trained for modular addition becomes centroid-aligned.
In particular, centroid alignment correlates with the generalization of the transformer from the training set to a test set.
We will now support this observation by connecting centroid alignment to the {\em feature learning regime} of training.

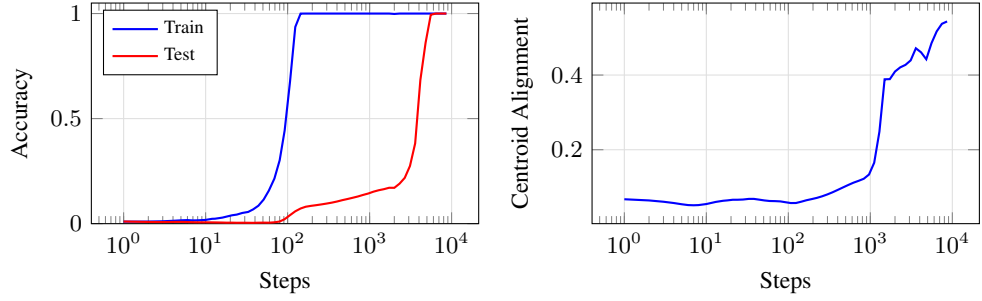
\begin{figure}[ht]
    \centering
    \begin{tikzpicture}
        \begin{groupplot}[
            group style={
                group size=2 by 1,
                horizontal sep=1.5cm,
            },
            width=0.48\linewidth,
            height=4.5cm,
            grid=major,
            grid style={solid, gray!25},
            tick label style={font=\small},
            label style={font=\small},
        ]
            \nextgroupplot[
                ylabel={Accuracy},
                ymin=0, ymax=1.05,
                legend pos=north west,
                legend cell align={left},
                xmode=log,
                xlabel={Steps},
                legend style={font=\scriptsize}
            ]
            \addplot [color=blue, thick] table [x=step, y=train_accuracy_mean, col sep=comma] {data/centroid_alignment_modadd_stats.csv};
            \addlegendentry{Train}
            \addplot [color=red, thick] table [x=step, y=test_accuracy_mean, col sep=comma] {data/centroid_alignment_modadd_stats.csv};
            \addlegendentry{Test}
            
            \nextgroupplot[
                ylabel={Centroid Alignment},
                legend pos=south east,
                xmode=log,
                xlabel={Steps},
            ]
            \addplot [color=blue, thick] table [x=step, y=centroid_alignment_mean, col sep=comma] {data/centroid_alignment_modadd_stats.csv};
        \end{groupplot}
    \end{tikzpicture}
    \caption{
    \textbf{A one-hidden-layer transformer training on modular arithmetic exhibits centroid alignment.}
    Here we train a one-layer transformer on a modular arithmetic task.
    On the left, we show the model's accuracy on the training and held-out test sets.
    On the right, we show the centroid alignment between the map from the embedding and the logits of the last token in the context.
    For more experimental details, see \Cref{exp:ca_modadd}.}
    \label{fig:ca_modadd}
\end{figure}

\subsubsection{Centroid Alignment Corresponds to Feature Learning}

Suppose a deep network $f$ has parameters (e.g., weights and biases) $\theta$. 
In particular, consider a two-layer deep network of the form $f_{\theta}(\bx)=\bW_2\left(\sigma\left(\bW_1\bx\right)\right)$, where $\bW_2\in\R^{d^{(2)}\times d^{(1)}}$, $\bW_1\in\R^{d^{(1)}\times d}$, and $\sigma$ is the ReLU nonlinearity.

\begin{lemma}\label{lem:centroid_two_layer_network}
    For $f_{\theta}$, we have $\bmu_{\varphi(\bx)}=\left(\bW_2\bQ_{\bx}\bW_1\right)^\top\mathbf{1}$, where $\bQ_{\bx}:=\mathrm{diag}\left(\sigma^\prime\left(\bW_1\bx\right)\right)$.
\end{lemma}

Suppose $\mathcal{L}=\frac{1}{n}\sum_{i=1}^n \ell\left(f\left(\bx_i\right),y_i\right)$ is an arbitrary differentiable loss function, and the deep network is being trained using full-batch gradient descent with a learning rate of $\eta$. 
We denote the negative gradient of the loss with respect to the network output as $\bm_{\bx_i}:= -\nabla_{\bz}\ell\left(\bz, y_i\right)\big|_{\bz=f_\theta(\bx_i)}$.

\begin{proposition}\label{prop:centroid_dynamics_general}
    In the setting described above, we have 
    \begin{equation*}
        \partial_t\left(\left\langle\bx,\bmu_{\varphi(\bx)}\right\rangle\right)=\frac{\eta}{n}\sum_{i=1}^n\bm_{\bx_i}^\top\left[\left(\bW_2\bQ_{\bx_i}\bQ_{\bx}\left(\bW_2\right)^\top\right)\left\langle\bx,\bx_i\right\rangle + \left(\sigma\left(\bW_1\bx\right)^\top\sigma\left(\bW_1\bx_i\right)\right)\right]\mathbf{1}.
    \end{equation*}
\end{proposition}

\textit{Proof.} See \Cref{proof:centroid_dynamics_general}. \qed

The neural tangent kernel~\cite{jacotNeuralTangentKernel2018} between $\bx,\bx^\prime\in\R^d$ is taken to be $\Theta\left(\bx,\bx^\prime\right)=\nabla_{\theta}f_{\theta}(\bx)\left(\nabla_{\theta}f_{\theta}\left(\bx^\prime\right)\right)^\top$.
For the two-layer network considered here, the neural tangent kernel is given by
\begin{equation*}
    \Theta(\bx,\bx_i)=\left(\sigma(\bW_1\bx)^\top\sigma(\bW_1\bx_i)\right)\bI+\left(\bx^\top\bx_i\right)\bW_2\bQ_{\bx_i}\bQ_{\bx}\left(\bW_2\right)^\top.
\end{equation*}
It follows that the expression of \Cref{prop:centroid_dynamics_general} can be written as
\begin{equation}\label{eq:centroid_innerproduct_ntk}
    \partial_t\left(\left\langle\bx,\bmu_{\varphi(\bx)}\right\rangle\right)=\frac{\eta}{n}\sum_{i=1}^n\bm_{\bx_i}^\top\Theta(\bx,\bx_i)\mathbf{1}.
\end{equation}

The \emph{linear} and \emph{feature} learning regimes of deep network training are characterized by having relatively static or dynamic neural tangent kernels, respectively~\cite{chizatLazyTrainingDifferentiable2019,woodworthKernelRichRegimes2020,moroshkoImplicitBiasDeep2020}.
More specifically, a deep network is in the linear learning regime when for $\bx\in\R^d$ we have $\partial_t\left(\boldsymbol{\Theta}\left(\bx,\bx_i\right)\right)=0$ for every $i=1,\dots,n$, and it is in the feature learning regime otherwise.
The former identifies when the deep network approximates a linear function, whereas the latter involves the deep network's nonlinearities.

\begin{theorem}\label{thm:centroid_alignment_feature_learning}
    Suppose that the deep network $f$ has memorized the training data (i.e., $\partial_t\left(\bm_{\bx_i}\right)=0$ for every $i=1,\dots,n$).
    Then, a changing rate of change of the centroid inner product (i.e., $\partial^2_t\left(\left\langle\bx,\bmu_{\varphi(\bx)}\right\rangle\right)\neq0$) implies the deep network is in the feature learning regime of training.
\end{theorem}

\textit{Proof.} See \Cref{proof:centroid_alignment_feature_learning}. \qed

\subsubsection{GrokAlign}\label{sec:grokalign}

From \Cref{thm:centroid_alignment_feature_learning}, it is apparent that encouraging centroid alignment is beneficial for inducing feature learning.
Thus, in this section, we explore how normal alignment (which implies centroid alignment) can be induced in deep networks trained via gradient-based methods.
\Cref{thm:optimal_jacobian} motivates the regularization of the Jacobian and offset norms during training, a method we introduce as GrokAlign.

However, a priori, it is unclear how this strategy is connected to, or more effective than, directly regularizing for normal alignment.
For simplicity, we henceforth consider bias-free models, such that $\bb_{\bx}=\boldsymbol{0}$ for every $\bx\in\R^d$.
However, all derivations also hold for biased models by adding a dimension of ones to the input space.

One reason the Jacobian matrices of a classifier may exhibit a high effective rank and fail to align under standard training is the presence of orthogonal components.
More specifically, if $\bJ_{\bx}\in\R^{C\times d}$ is not aligned, then necessarily there exists non-zero vectors $\bc\in\R^C$ and $\bu\in\R^d$ along with a matrix $\tilde{\bJ}_{\bx}\in\R^{C\times d}$ such that $\bJ_{\bx}=\tilde{\bJ}_{\bx}+\bc\bu^\top$ with $\bu\perp\bx$.
This orthogonal component $\bc\bu^\top$ of $\bJ_{\bx}$ does not influence the output of the classifier as $f(\bx)=\bJ_{\bx}\bx=\tilde{\bJ}_{\bx}\bx+\bc\bu^\top\bx=\tilde{\bJ}_{\bx}\bx$.
Therefore, during training, the classifier will not observe any gradients to remove this component from the Jacobian.

To remove these orthogonal components, it is necessary for the Jacobian of $f$ at $\bx$ to operate on an input that is not parallel to $\bx$.
For a continuous piecewise affine (CPA) model, we can apply a small perturbation $\boldsymbol{\epsilon}\in\R^d$ to $\bx$ such that $\bJ_{\bx+\boldsymbol{\epsilon}}=\bJ_{\bx}$.
Hence, 
\begin{equation}\label{eq:orthogonal_contribution}
    f(\bx+\boldsymbol{\epsilon})=\bJ_{\bx+\boldsymbol{\epsilon}}(\bx+\boldsymbol{\epsilon})=\bJ_{\bx}\bx+\bJ_{\bx}\boldsymbol{\epsilon}=f(\bx)+\left(\tilde{\bJ}_{\bx}+\bc\bu^\top\right)\boldsymbol{\epsilon},
\end{equation}
meaning the orthogonal components we want to remove contribute to the classifier's output.
Therefore, in practice, to remove the orthogonal components, it suffices to regularize $\mathcal{R}_{\sigma}(\bx):=\mathbb{E}_{\boldsymbol{\epsilon}}\left(\frac{\left\Vert f(\bx+\boldsymbol{\epsilon})-f(\bx)\right\Vert_2^2}{\sigma^2}\right)$ for $\boldsymbol{\epsilon}$ drawn from a Gaussian distribution with covariance $\sigma^2\boldsymbol{I}$.
As $\sigma\to0$ it follows that $\mathcal{R}_{\sigma}(\bx)\to\left\Vert\bJ_{\bx}\right\Vert_F^2=:\mathcal{R}(\bx)$; meaning, GrokAlign regularization can be interpreted as a method for removing the orthogonal components of a classifier's Jacobians.
In \Cref{sec:ablation}, we review the practical implementation of GrokAlign and compare it with directly regularizing for normal alignment.

\begin{table}[t]
    \centering
    \caption{
    \textbf{GrokAlign significantly accelerates the rate of grokking.} 
    For each setting, we consider the training pipelines across ten random initializations and record the epoch at which the grokked state is reached. 
    We report the mean epoch value, the rate of improvement over the baseline training pipeline, and the $p$-value of the corresponding paired t-test comparing the individual training runs of the considered regularization and the baseline -- insignificant results are denoted with an asterisk. 
    For more experimental details, see \Cref{exp:acceleration}.
    % Code to reproduce these results and implement GrokAlign can be found \href{https://anonymous.4open.science/r/GrokAlign-F8E5/README.md}{here.}
    }
    \label{tab:acceleration}
    \resizebox{\textwidth}{!}{%
    \begin{tabular}{llcccc}
    \toprule
    \textbf{Dataset} & \textbf{Metric} & \textbf{Baseline} & \textbf{Grokfast} & \textbf{OrthoGrad} & \textbf{GrokAlign} \\
    \midrule
    
    XOR & Reaches Grokked State & $100\%$ & $100\%$ & $20\%$ & $100\%$ \\
     & Number of Epochs & $148$ & $148$ & $59^*$ & $\mathbf{97}$ \\
     & Rate of Speed-Up & -- & $1.0$ & $2.51^*$ & $\mathbf{1.51}$ \\
     & $p$-value & -- & -- & $0.053$ & $3.2\times10^{-8}$ \\
    \midrule

    Sparse Parity & Reaches Grokked State & $100\%$ & $90\%$ & $100\%$ & $100\%$ \\
     & Number of Epochs & $2000$ & $1629^*$ & $242$ & $\mathbf{101}$ \\
     & Rate of Speed-Up & -- & $1.23^*$ & $8.26$ & $\mathbf{19.76}$ \\
     & $p$-value & -- & $0.06$ & $1.43\times10^{-8}$ & $1.11\times10^{-8}$ \\
    \midrule
    
     MNIST - Cross Entropy & Reaches Grokked State & $100\%$ & $100\%$ & $100\%$ & $100\%$ \\
     & Number of Epochs & $2070$ & $2230$ & $2570$ & $\mathbf{329}$ \\
     & Rate of Speed-Up & -- & -- & -- & $\mathbf{6.29}$ \\
     & $p$-value & -- & -- & -- & $6.4\times10^{-5}$ \\
    \midrule
    
     MNIST - Squared Error & Reaches Grokked State & $100\%$ & $10\%$ & $0\%$ & $100\%$ \\
     & Number of Epochs & $7630$ & $7300^*$ & & $\mathbf{1170}$ \\
     & Rate of Speed-Up & -- & $1.05^*$ & -- & $\mathbf{6.52}$ \\
     & $p$-value & -- & $0.32$ & -- & $6.6\times10^{-10}$ \\
    \midrule

     Modular Addition & Reaches Grokked State & $100\%$ & $100\%$ & $0\%$ & $100\%$ \\
      & Number of Epochs & $265$ & $251^*$ & -- & $\mathbf{166}$ \\
      & Rate of Speed-Up & -- & $1.06^*$ & -- & $\mathbf{1.60}$ \\
      & $p$-value & -- & $0.11$ & -- & $1.8\times10^{-10}$ \\
    
    \bottomrule
    \end{tabular}
    }
\end{table}

\subsubsection{Accelerating Grokking with GrokAlign}

Grokking is the phenomenon in deep network training where training accuracy can saturate relatively quickly, and it takes a significant amount of further training for performance on test data to improve~\cite{powerGrokkingGeneralizationOverfitting2022}.
This is a prototypical example of a sparse problem, as, by definition, the model's loss value can be changed independently across different samples.

The prominent explanation for grokking is that the deep network is stuck in the linear learning regime at the beginning of training, and the induction of feature learning is delayed~\cite{lyuDichotomyEarlyLate2024,rubinGrokkingFirstOrder2024,kumarGrokkingTransitionLazy2024}.
Consequently, with \Cref{thm:centroid_alignment_feature_learning}, GrokAlign should accelerate grokking.
We explore this in this section.

We compare the effectiveness of GrokAlign at inducing grokking to two other methods designed to accelerate grokking against a baseline. 
Grokfast~\cite{leeGrokfastAcceleratedGrokking2024} works to accelerate the rate of grokking by manipulating the gradients during training to amplify certain signals. 
OrthoGrad~\cite{prietoGrokkingEdgeNumerical2025} aligns gradients to prevent na\"{i}ve loss minimization and encourage generalization.

We apply these methods to fully connected deep networks learning an XOR task~\cite{xu2025let}, MNIST~\cite{liuOmnigrokGrokkingAlgorithmic2022}, modular addition~\cite{mallinarEmergenceNonneuralModels2025}, and a sparse parity task~\cite{prietoGrokkingEdgeNumerical2025}.
Across multiple random initializations, we measure the number of epochs required to reach the grokked state, as specified in \Cref{tab:grokking_criteria}.
To determine statistical significance against the baseline, we perform a paired t-test on the number of epochs required to grok.

From \Cref{tab:acceleration}, we see that GrokAlign provides the most significant acceleration of grokking.
In particular, it performs consistently across settings, whereas the other regularization strategies exhibit variable performance. 

\subsection{Recursive Feature Machines}

Recursive Feature Machines (RFMs) are an integration of feature learning principles as formalized by the average gradient outer product, into classical kernel-based machine learning models~\cite{radhakrishnanMechanismFeatureLearning2024}.

Given a dataset $\mathcal{D}$, a RFM is given by $f(\bx)=\sum_{i=1}^n\balpha_i\phi_M\left(\bx_i,\bx\right)$, where $\balpha\in\R^n$ and $\phi_M$ is a kernel function that incorporates a learnable feature matrix $\bM\in\R^{d\times d}$.
For simplicity, and in line with \citet{radhakrishnanMechanismFeatureLearning2024}, we consider $\phi_M$ to be the Laplace kernel $\phi_{\bM}(\bx,\bz)=\exp\left(-\gamma\left\Vert\bx-\bz\right\Vert_M\right)$ where $\gamma>0$ and $\left\Vert\bx-\bz\right\Vert_2^2=(\bx-\bz)^\top\bM(\bx-\bz)$.
The method for training an RFM computes the Gram matrix $\left(\bJ_{\bx}\right)^\top\bJ_{\bx}$ for training points and is described in \Cref{alg:training_kernel_methods}.

\begin{algorithm}[ht]
\caption{RFM Training.}\label{alg:training_kernel_methods}
\begin{algorithmic}
\Require Training data $\mathcal{D}=\left\{\left(\bx_i,y_i\right)\right\}_{i=1}^n$, Iterations $T$
\State $\bM\leftarrow\bM_{\text{init}}$
\For{$t=1,\dots,T$}
\State $\bK\leftarrow\sum_{i,j=1}^{n}\phi_{\bM}\left(\bx_i,\bx_j\right)\be_i\be_j^{\top}$
\State $\balpha=\bK\by$\Comment{$\by=\left(y_1,\dots,y_n\right)^{\top}$}
\State $\bM\leftarrow\frac{1}{n}\sum_{i=1}^n\left(\bJ_{\bx_i}\right)^\top\bJ_{\bx_i}$\Comment{$f(\bx)=\sum_{i=1}^n\balpha_i\phi_{\bM}\left(\bx_i,\bx\right)$}
\EndFor
\end{algorithmic}
\end{algorithm}

\subsubsection{Fixed Point Properties of Recursive Feature Machines}

In \citet{radhakrishnanMechanismFeatureLearning2024}, $\bM_{\text{init}}$ of \Cref{alg:training_kernel_methods} takes the form of the $d\times d$ identity matrix.
This assumes there is no a priori understanding of which features the model should realize.
However, from our discussion, we would expect the Jacobians to align with the training data.
Supposing that $\bJ_{\bx_i}=\bc_i\bx_i^\top$, it follows that 
\begin{equation*}
    \bM=\frac{1}{n}\sum_{i=1}^n\left(\bJ_{\bx_i}\right)^{\top}\bJ_{\bx_i}=\frac{1}{n}\sum_{i=1}^n\left\Vert\bc_i\right\Vert_2^2\bx_i\bx_i^{\top}=\sum_{i=1}^nc_i\bx_i\bx_i^{\top}.
\end{equation*}
That is, the feature matrix becomes a linear combination of the outer products of the training samples.
This property of a feature matrix $\bM$ is preserved under the iterations of \Cref{alg:training_kernel_methods}.

\begin{proposition}\label{prop:feature_matrix_linear_combination}
    Suppose $f$ is a RFM with $\bM=\sum_{i=1}^nc_i\bx_i\bx_i^\top$.
    Then, one iteration of \Cref{alg:training_kernel_methods} yields a RFM with feature matrix $\bM^\prime=\sum_{i=1}^nc_i^\prime\bx_i\bx_i^\top$ for some $c_i^\prime\in\R$.
\end{proposition}

\textit{Proof.} See \Cref{proof:feature_matrix_linear_combination}. \qed

\subsubsection{Recursive Feature Alignment Machines}

While grokking and modular arithmetic represent explicit examples of sparse learning environments, tabular data presents a uniquely pervasive case of implicit sparsity in machine learning. 
Tabular datasets typically fall into the sparse regime due to the breakdown of the MH.
Indeed, tabular datasets are fundamentally anisotropic; their dimensions represent features with distinct semantics, scales, and distributions.
Thus, defining a natural, continuous metric space or assuming a smooth underlying geometry between samples is challenging.
This geometric disconnect is compounded by the typically low sample-to-feature-complexity ratio inherent to tabular tasks.
Therefore, the theory of normal alignment seems particularly applicable to models training on tabular data.

As a result of \Cref{prop:feature_matrix_linear_combination}, we consider a generalized strategy for setting $\bM_{\text{init}}$ in \Cref{alg:training_kernel_methods}, to yield Recursive Feature Alignment Machines (RFAMs).
RFAMs are trained using \Cref{alg:training_kernel_methods} with $\bM_{\text{init}}$ to $(1-\alpha)\mathrm{Cov}\left(\bX\right)+\alpha\bI$, where $\bX:=\left(\bx_1,\dots,\bx_n\right)^\top\in\R^{n\times d}$ is the data matrix, $\mathrm{Cov}(\bX)$ is its covariance matrix, $\bI\in\R^{d\times d}$ is the identity matrix, and $\alpha\in[0,1]$.
Hence, with $\alpha$ equal to one, we recover RFMs.

To compare RFMs with RFAMs, we consider the tabular tasks of \citet{fernandez-delgadoWeNeedHundreds2014} and \citet{erickson2025tabarena}.
For the tasks reported by \citet{fernandez-delgadoWeNeedHundreds2014}, we compare RFMs ($\alpha=1.0$) and RFAMs with $\alpha=0.0$.
In \Cref{tab:rfms}, we see that RFAMs are analogous to an ``adversarially trained'' RFMs.
Although their test accuracy is lower than RFMs, they exhibit greater robustness.
This resembles the accuracy-robustness trade-off present in DNs~\cite{tsiprasRobustnessMayBe2019}.

\begin{table}[ht]
    \centering
    \caption{
    \textbf{RFAMs yield significantly more robust models than RFMs on tabular tasks.}
    Using the tabular datasets of \citet{fernandez-delgadoWeNeedHundreds2014}, we compare the test accuracy and robustness of RFAMs and RFMs.
    Robustness is measured as the proportion of correctly classified test samples that are successfully perturbed by a PGD-perturbation~\cite{madryDeepLearningModels2018} of amplitude $1.0$ to become misclassified by the model.
    Normal alignment is measured across the training set.
    All results are significant at a $1\%$ level using a t-test.
    For more experimental details and results, see \Cref{exp:rfms}.
    % Code to reproduce these results and implement RFAMs can be found \href{https://anonymous.4open.science/r/RFAM-BB6E/README.md}{here.}
    }
    \vspace{0.5em}
    \begin{tabular}{ccccc}
        \toprule
        Method & Test Accuracy $(\uparrow)$ & Attack Success Rate $(\downarrow)$ & Normal Alignment $(\uparrow)$ \\
        \midrule
        RFMs & $\mathbf{84.9\%}$ & $71.6\%$ & $0.31$ \\
        RFAMs & $83.2\%$ & $\mathbf{68.3\%}$ & $\mathbf{0.38}$ \\
        \bottomrule
    \end{tabular}
    \label{tab:rfms}
\end{table}

For the tasks of \citet{erickson2025tabarena}, we consider values of $\alpha$ in the range $\{0.0,0.001,0.01,0.1,1.0\}$.
We find that $\alpha$ values less than $1.0$ improve performance in a third of the cases, by $\mathbf{5.27}\%$ on average in terms of validation error.
The breakdown of the validated $\alpha$ values across these tasks is $34$, $11$, $0$, $2$ and $4$ for $\alpha$ values $1.0$, $0.1$, $0.01$, $0.001$ and $0.0$, respectively.
In \Cref{tab:tabarena}, we provide a detailed breakdown of the individual improvements.

\section{Discussion}\label{sec:discussion}

This work provides a rigorous geometric framework to explain how models succeed when classical assumptions, like the manifold hypothesis, fail.
We show that normal alignment serve as a structural signature of models trained in these sparse settings, offering a mechanistic explanation for how models successfully navigate data-scarce or discrete environments.

We actively translate this theoretical framework into actionable algorithms. 
By regularizing for normal alignment via GrokAlign, we consistently accelerates grokking dynamics across a diverse set of tasks. 
Similarly, by framing tabular datasets as fundamentally sparse, Recursive Feature Alignment Machines (RFAMs) leverage these geometric principles to achieve superior adversarial robustness over standard Recursive Feature Machines (RFMs).

A key direction for future work is mapping the phase transition between sparse and ``dense'' learning regimes, particularly since aggressive data augmentation diminishes a model's reliance on explicit alignment. 
Additionally, extending normal alignment regularization to domains dominated by discrete tokens and sparse signals, such as reinforcement learning and large language models, represents a promising frontier for accelerating feature learning at scale.

\section*{Reproducibility Statement}

For an implementation of GrokAlign and the code to reproduce the results of \Cref{tab:acceleration}, refer to the following repository: \url{https://github.com/ThomasWalker1/GrokAlign}.
For an implementation of RFAMs and the code to reproduce the results of \Cref{tab:rfms}, refer to the following repository: \url{https://github.com/ThomasWalker1/RFAM}.

\section*{Acknowledgments}

This work was supported by ONR grant N00014-23-1-2714, DOE grant DE-SC0020345, DOI grant 140D0423C0076, and a Google Cloud Computing Award.

\newpage
\bibliographystyle{unsrtnat}
\bibliography{references}
\newpage
\appendix

\section{Constructing Normal Aligned Deep Networks}\label{sec:constructing_normal_aligned_deep_networks}

For simplicity, consider each $\bx_i$ to be unit norm and let $f(\bx)=\bW_2\sigma\left(\bW_1\bx+\bb\right)$ where $\bW_1\in\R^{n\times d}$, $\bb\in\R^n$, $\bW_2\in\R^{C\times n}$ and $\sigma$ is the ReLU activation function.
Let $m_i:=\max_{i\neq j}\left(\left\langle\bx_i,\bx_j\right\rangle\right)$.
Then setting $\bW_1^{(i)}=\bx_i$ and $\bb_i:=-\frac{1+m_i}{2}$ is sufficient to yield a normal aligned deep network.
Intuitively, the deep network is constructed by positioning the activation level sets of each neuron (i.e., the hyperplane along which the input to the nonlinearity is zero) such that the $i^\text{th}$ neuron is only active for $\bx_i$ and the normal of the hyperplane is parallel to the direction $\bx_i$.
A visualization of this procedure is shown in \Cref{fig:construction}.

Since this construction is only dependent on $\bW_1$ and $\bb$, the parameters $\bW_2$ do not contribute to the alignment property.
Instead, $\bW_2$ manipulates the output of the $i^{\text{th}}$ neuron, which is $\bz_i:=\frac{1-m_i}{2}$, to form the output of the model.
If $\mathcal{D}$ is a {\em denser} sample, then $m_i$ is closer to one, which means that $\bz_i$ approaches zero.
Hence, for a fixed output, $\bW_2$ must have a larger norm, which means that under regularity constraints (e.g., weight decay) the normal-aligned solution is more challenging to learn.
In \Cref{fig:construction} we visualize this by coloring the scatter points in the second and fourth panel according to the norm of $\bW_2$ necessary to ensure that $f(\bx_i)=1$.

\section{Normal Alignment For Specific Loss Functions}\label{sec:alignment_loss_functions}

For specific loss functions, we can characterize the $\bc_i$ of \Cref{thm:optimal_jacobian}.
For the squared-error loss function, the $c^\text{th}$ row of the Jacobian at the training data point $\bx_i$ is given by (left), while for the cross-entropy loss function, it is given by (right):

\begin{minipage}{0.45\textwidth}
\begin{equation}\label{eq:ci_se}
    \bJ_{\bx_i}^{(c)}=\begin{cases}\beta\bx_i & c=y_i\\\boldsymbol{0}&c\neq y_i,\end{cases}
\end{equation}
\end{minipage}
\begin{minipage}{0.45\textwidth}
\begin{equation}\label{eq:ci_ce}
    \bJ_{\bx_i}^{(c)}=\begin{cases}\gamma\sqrt{C-1}\bx_i&c=y_i\\-\frac{\gamma}{\sqrt{C-1}}\bx_i & c\neq y_i,\end{cases}
\end{equation}
\end{minipage}

with $\beta,\gamma$ positive constants dependent on $\bx_i$.

In particular, the normal alignment of \Cref{eq:ci_se} yields centroids $\bmu_{\bx_i}$ which are projections of the training data onto a hyper-sphere, whereas \Cref{eq:ci_ce} yields centroids that are zero.
This latter case is not an issue, since it simply implies that the deep network's output on the training data is linear across the last few intermediate hidden layers.
Indeed, with \Cref{fig:resnet_alignment}, we observe that the maps from the input space of intermediate layers to the output space of residual neural networks~\cite{heDeepResidualLearning2016} exhibit centroid alignment.

\section{Gaussian Kernel Logistic Regression Models}\label{sec:gaussian_kernel_logistic_regression_models}

As a preliminary example, we consider a Gaussian kernel logistic regression model.
Namely, for $c=1,\dots,C$ we have $f_c(\bx):=\left[f(\bx)\right]_c=\sum_{k=1}^K\bW_{ck}\phi_k(\bx)$ for $\bW\in\R^{C\times K}$, $\phi_k(\bx)=\exp\left(-\gamma\left\Vert\bx-\boldsymbol{\tau}_i\right\Vert_2^2\right)$, $\boldsymbol{\tau}_i\in\R^d$, and $\gamma\in\R$. 
The parameters of the model are the weights $\bW$ and the centers $\left\{\boldsymbol{\tau}\right\}_{k=1}^K$, whereas $K$ and $\gamma$ are hyper-parameters.

\begin{lemma}\label{lem:gaussian_jacobian}
    For $f$ a Gaussian kernel logistic model, $\bJ_{\bx}^{(c)}=-2\gamma\sum_{k=1}^K\bW_{ck}\phi_k(\bx)\bx+2\gamma\sum_{k=1}^K\bW_{ck}\phi_k(\bx)\boldsymbol{\tau}_k$.
\end{lemma}

\textit{Proof.} See \Cref{proof:gaussian_jacobian}. \qed

\Cref{lem:gaussian_jacobian} shows that a Gaussian kernel logistic model's Jacobians have two components, one of which is aligned to the input point, and another which is a weighted sum of the model's centers.
Therefore, normal alignment emerges either when this second component is zero or is aligned with the input point.
In practice, see \cref{fig:gaussian_log_reg}, we see that the model progresses toward the normal aligned state, with its effective rank collapsing towards one.

\section{Models with Jacobians Equal to Zero}\label{sec:zero_jacobians}

Normal aligned classifiers are optimally robust amongst classifiers whose input-output Jacobians are non-zero.
Here, we demonstrate theoretically that this class of classifiers emerges under natural regularized training.
Furthermore, we show that classifiers with zero Jacobians suffer from performance degradation.

Suppose that weight decay is applied with a coefficient $\eta>0$.
Then there exists a $\mathcal{B}(\eta)$ for the $\ell_2$-norm of $\bb_{\bx}$ with the property that $\mathcal{B}(\eta)\to0$ as $\eta\to\infty$.
Thus,
\begin{equation*}
    \left\vert\left[\bb_{\bx}\right]_y-\left[\bb_{\bx}\right]_c\right\vert\leq\sqrt{2}\left\Vert\bb_{\bx}\right\Vert_2\leq\sqrt{2}\mathcal{B}(\eta).
\end{equation*}
Now \Cref{thm:optimal_jacobian} already demonstrates that there is an optimization pressure to have non-zero Jacobians.
However, we can deduce this more naturally by noting that a sufficiently well-trained classifier will satisfy some classification margin condition.
Namely, $f_y(\bx)-f_c(\bx)\geq\gamma$ for some $\gamma>0$.
Consequently,
\begin{equation*}
    \left(\bJ_{\bx}^{(y)}-\bJ_{\bx}^{(c)}\right)^\top\bx\geq\gamma-\left(\left[\bb_{\bx}\right]_y-\left[\bb_{\bx}\right]_c\right).
\end{equation*}
Assuming $\bx$ is non-zero, some rearranging allows us to conclude that $\bJ_{\bx}$ is non-zero as
\begin{equation*}
    \left\Vert\bJ_{\bx}^{(y)}-\bJ_x^{(c)}\right\Vert_2\geq\frac{\gamma-\sqrt{2}\mathcal{B}(\eta)}{\Vert\bx\Vert_2}.
\end{equation*}
Hence, for each $c\neq y$ we have
\begin{equation*}
    \frac{\left\vert\left[\bb_{\bx}\right]_y-\left[\bb_{\bx}\right]_c\right\vert}{\left\Vert\bJ_{\bx}^{(y)}-\bJ_{\bx}^{(c)}\right\Vert_2}\leq\frac{\sqrt{2}\mathcal{B}(\eta)\Vert\bx\Vert_2}{\gamma-\sqrt{2}\mathcal{B}(\eta)}=:\lambda_c.
\end{equation*}
Therefore, taking $\lambda=\max_{c\neq y}\lambda_c$ is sufficient to satisfy the ratio constraint of \Cref{sec:properties}.

We can similarly use this argument to show that explicitly regularizing for zero Jacobians can degrade the classifier's performance.
Indeed, if the Jacobians were zero, then $f(\bx)=\bb_x$, which implies the margin is given by $\min_{c\neq y}\left(\left[\bb_{\bx}\right]_y-\left[\bb_{\bx}\right]_c\right)$.
However, due to weight-decay, this implies that the margin is bounded by $\mathcal{B}(\eta)$.
Since larger margins are associated with better classifiers, this implies that classifiers with zero Jacobians perform worse.

We can empirically verify that in practice, models do not learn classifiers with zero Jacobians, and that regularizing for zero Jacobians affects performance.
We train a fully connected deep network on a subset of MNIST~\cite{lecunGradientBasedLearningApplied1998} of size $1000$.
We train the network adversarially by applying PGD~\cite{madryDeepLearningModels2018} to each training batch and weight decay.
Furthermore, we apply a Jacobian Frobenius norm penalty to the loss function with a weighting factor $\gamma$.
Throughout training, we monitor the model's clean, robust accuracy on a held-out test set, as well as the norms of the Jacobians and offset terms evaluated on the training data.

In \Cref{fig:gamma_analysis}, we see that for small values of $\gamma$ the norms of the Jacobians increase during training and the size of the offset terms converge to a bounded value.
It is only for large values of $\gamma$ for which solutions with Jacobians equal to zero are learned.
However, in these cases, the model's performance severely degrades; perhaps because the corresponding offset terms still converge to bounded values.

Thus, it appears that the optimally robust solution of a classifier with Jacobians equal to zero suffers from not being optimal for the task.
This is a common trade-off when training adversarially robust models.
An optimally robust model is just the constant function; however, such a function is not able to learn the task~\cite{tsiprasRobustnessMayBe2019}.

\section{Gradient-Based Regularization Ablation}\label{sec:ablation}

In \Cref{sec:grokalign}, we introduce the GrokAlign method for regularizing for normal alignment.
In this section, we compare it with other forms of regularization and examine its practical implementation.
With \Cref{sec:direct_regularization,sec:nuclear_norm_regularization} we demonstrate how direct optimization for normal alignment and nuclear norm regularization are similar to GrokAlign, $\mathcal{R}$, in that they involve the regularization of the Frobenius norm a model's Jacobian.
Thus, with \Cref{sec:implementing_frob_norm_regularization} we consider the practical implementation of computing the Frobenius norm of a model's Jacobian.
Therefore, in \Cref{sec:empirical_comparison}, we can empirically compare the performance of these different forms of regularization at inducing normal alignment.

\subsection{Directly Optimizing for Alignment}\label{sec:direct_regularization}

Let $\hat{\bu}\in\R^d$ be a unit vector.
Let $\bP_{\hat{\bu}^{\perp}}=\bI-\hat{\bu}\hat{\bu}^\top$ be the orthogonal projector onto the subspace orthogonal to $\hat{\bu}$. Let $\mathcal{R}_{\perp}(\bx)=\left\Vert\bJ_{\bx}\bP_{\hat{\bu}^{\perp}}\right\Vert_F^2$.
Intuitively, $\mathcal{R}_{\perp}(\bx)$ measures how far $\bJ_{\bx}$ deviates from spanning the direction of $\hat{\bu}$.

\begin{proposition}\label{prop:alignment_frob-proj_equivalence}
    With notation as above, $\min_{\bc\in\R^C}\left\Vert\bJ_{\bx}-\bc\hat{\bu}^\top\right\Vert_F^2=\mathcal{R}_{\perp}(\bx)$ with minimizer $\bc^{\star}=\bJ_{\bx}\hat{\bu}$.
\end{proposition}

\begin{proof}
    Observe that
    \begin{align*}
        \left\Vert\bJ_{\bx}-\bc\hat{\bu}^\top\right\Vert_F^2&=\left\Vert\bJ_{\bx}\left(\bP_{\hat{\bu}^\perp}+\hat{\bu}\hat{\bu}^\top\right)-\bc\hat{\bu}^\top\right\Vert_F^2\\&=\left\Vert\bJ_{\bx}\bP_{\hat{\bu}^\perp}+\left(\bJ_{\bx}\hat{\bu}-\bc\right)\hat{\bu}^\top\right\Vert_F^2\\&=\left\Vert\bJ_{\bx}\bP_{\hat{\bu}^\perp}\right\Vert_F^2+\left\Vert\left(\bJ_{\bx}\hat{\bu}-\bc\right)\hat{\bu}^\top\right\Vert_F^2+2\left\langle\bJ_{\bx}\bP_{\hat{\bu}^\top},\left(\bJ_{\bx}\hat{\bu}-\bc\right)\hat{\bu}^\top\right\rangle_F\\&\overset{(1)}{=}\left\Vert\bJ_{\bx}\bP_{\hat{\bu}^\perp}\right\Vert_F^2+\left\Vert\left(\bJ_{\bx}\hat{\bu}-\bc\right)\hat{\bu}^\top\right\Vert_F^2\\&=\left\Vert\bJ_{\bx}\bP_{\hat{\bu}^\perp}\right\Vert_F^2+\left\Vert\bJ_{\bx}\hat{\bu}-\bc\right\Vert_2^2,
    \end{align*}
    where in (1) we have used the fact that $\bJ_{\bx}\bP_{\hat{\bu}^\perp}$ has columns orthogonal to $\hat{\bu}$ and $\left(\bJ_{\bx}\hat{\bu}-\bc\right)\hat{\bu}^\top$ has columns in the span of $\hat{\bu}$.
    Thus, the result follows.
\end{proof}

From \cref{prop:alignment_frob-proj_equivalence} it follows that regularizing $\mathcal{R}_{\perp}(\bx)$ is equivalent to fitting the Jacobian $\bJ_{\bx}$ to a matrix aligned to the vector $\hat{\bu}$.
Thus, to attain normal alignment, it is sufficient to replace $\hat{\bu}$ with $\frac{\bx}{\Vert\bx\Vert}$.

\subsection{Nuclear Norm Regularization.}\label{sec:nuclear_norm_regularization}

The normal-aligned solution represent a deep network with rank-one Jacobians at the training.
Since, the nuclear norm constraint is the convex relaxation of minimizing rank~\cite{rechtGuaranteedMinimumrankSolutions2010}, it seems appropriate to consider nuclear norm regularization of a deep network's Jacobian to induce normal alignment.

In \citet{scarvelisNuclearNormRegularization2024}, it is shown that regularization of the Frobenius norms of the Jacobians of two sub-components of a classifier, whose composition yields the full input-output mapping, is equivalent to regularizing the nuclear norm of the full input-output Jacobian of the DN.

\begin{theorem}[\citealt{scarvelisNuclearNormRegularization2024}]
    Suppose $f=g\circ h$.
    Then minimizing $\ell(f(\bx),y)+\eta\left\Vert\bJ_{\bx}(f)\right\Vert_{\star}$, where $\Vert\cdot\Vert_{\star}$ denotes nuclear norm, is equivalent to minimizing $\ell(f(\bx),y))+\eta\mathcal{R}_{\text{Nuc}}(\bx)$, where $\mathcal{R}_{\text{Nuc}}(\bx):=\frac{1}{2}\left(\left\Vert\bJ_{h(\bx)}(g)\right\Vert_F^2+\left\Vert\bJ_{\bx}(h)\right\Vert_F^2\right)$.
\end{theorem}

\subsection{Practical Implementation of Frobenius Norm Regularization}\label{sec:implementing_frob_norm_regularization}

We have demonstrated that we can reduce the practical implementation of $\mathcal{R}$, $\mathcal{R}_{\perp}$ and $\mathcal{R}_{\text{Nuc}}$ to understanding how to compute the Frobenius norms of a model's Jacobians.
In particular, it is unnecessary to compute the full Jacobians of the classifier.
Computing Jacobians of classifiers is computationally expensive.
Similar to how stochasticity makes gradient descent tractable in practice (i.e., stochastic gradient descent), we can use stochasticity to make the above regularizers tractable in practice.
We leverage the fact that $\mathrm{tr}\left(\bJ_{\bx}\bJ_{\bx}^\top\right)=\left\Vert\bJ_{\bx}\right\Vert_F^2$ to note that we can use Hutchinson's estimator to form an unbiased estimate of $\left\Vert\bJ_{\bx}\right\Vert_F^2$.

\begin{lemma}\label{lem:hutchinson}
    Let $\bJ\in\R^{C\times d}$, and let $\bz\in\R^d$ satisfy $\mathbb{E}(\bz)=\boldsymbol{0}$ and $\mathbb{E}\left(\bz\bz^\top\right)=\bI_d$.
    Then, $\mathbb{E}\left(\left\Vert\bJ\bz\right\Vert_2^2\right)=\mathrm{tr}\left(\bJ\bJ^\top\right)=\left\Vert\bJ\right\Vert_F^2$.
    Equally, if $\bu\in\R^C$ satisfies $\mathbb{E}\left(\bu\right)$ and $\mathbb{E}\left(\bu\bu^\top\right)=\boldsymbol{0}=\bI_C$.
    Then, $\mathbb{E}\left(\left\Vert\bJ^\top\bu\right\Vert_2^2\right)=\mathrm{tr}\left(\bJ\bJ^\top\right)=\left\Vert\bJ\right\Vert_F^2$.
\end{lemma}

The implementation of \Cref{lem:hutchinson} for estimating the Frobenius norm of $\bJ_{\bx}$ can be found in \Cref{alg:input_approximation}.
However, \Cref{lem:hutchinson} can equivalently be stated using random vectors in $\R^C$.
The implementation of this estimation of the Frobenius norm of $\bJ_{\bx}$ can be found in \Cref{alg:output_approximation}.

These estimators are sufficient for implementing GrokAlign, $\mathcal{R}$, and nuclear norm regularization, $\mathcal{R}_{\text{Nuc}}$.
In practice, we utilize random vectors in the output space of classifiers, to obtain an approximation of the Frobenius norm of the classifier's Jacobian, (i.e., \Cref{alg:output_approximation}).
The reason for this is that we can use the accumulated gradients from the forward pass to compute the Jacobian vector product $\bJ_{\bx}^\top\bu$ with \texttt{torch.autograd.grad}.
Using random vectors in the input space (i.e., \Cref{alg:input_approximation}) would require the use of \texttt{torch.nn.functional.jvp} which performs an additional forward pass through the classifier.

\begin{algorithm}[ht]
\caption{Estimating Frobenius norms of Jacobians using input random vectors.}\label{alg:input_approximation}
\begin{algorithmic}
\Require Classifier $f$, input $\bx\in\R^d$, Number of projections $K$.
\For{$k=1,\dots,K$}
\State Sample $\bu$ from $\mathcal{N}\left(\boldsymbol{0},\bI_d\right)$
\State $\bJ^{(\bu)}_{\bx}\leftarrow\bJ_{\bx}\bu$ \Comment{jvp, requires a forward pass.}
\State $s_k\leftarrow\left\Vert\bJ^{(\bu)}_{\bx}\right\Vert_2^2$
\EndFor
\State $\mathcal{R}(\bx)\leftarrow\frac{1}{K}\sum_{k=1}^Ks_k$\\
\Return $\mathcal{R}(\bx)$
\end{algorithmic}
\end{algorithm}

\begin{algorithm}[ht]
\caption{Estimating Frobenius norms of Jacobians using output random vectors.}\label{alg:output_approximation}
\begin{algorithmic}
\Require Classifier $f$, input $\bx\in\R^d$, Number of projections $K$.
\State Accumulate gradients $f(\bx)$ \Comment{Already necessary for computing classification loss.}
\For{$k=1,\dots,K$}
\State Sample $\bu$ from $\mathcal{N}\left(\boldsymbol{0},\bI_C\right)$
\State $\bJ^{(\bu)}_{\bx}\leftarrow\bJ_{\bx}^{\top}\bu$ \Comment{autograd, uses already accumulated gradients.}
\State $s_k\leftarrow\left\Vert\bJ^{(\bu)}_{\bx}\right\Vert_2^2$
\EndFor
\State $\mathcal{R}(\bx)\leftarrow\frac{1}{K}\sum_{k=1}^Ks_k$\\
\Return $\mathcal{R}(\bx)$
\end{algorithmic}
\end{algorithm}

The benefit of considering random vectors in the input space, is that the random vectors live in the same space as $\bx$.
This is useful, as an extension of \cref{lem:hutchinson} lets us similarly approximate $\left\Vert\bJ_{\bx}\bP_{\hat{\bx}^{\perp}}\right\Vert_F^2$ to get a practical way to implement the regularizer $\mathcal{R}_{\perp}(\bx)$.

\begin{lemma}\label{lem:hutchinson_projected}
    Let $\hat{\bu}\in\R^d$ be of unit length, and $\bP_{\hat{\bu}^{\perp}}=\bI-\hat{\bu}\hat{\bu}^\top$.
    Let $\bz\in\R^d$ be as in \cref{lem:hutchinson}, and $\bz_{\perp}=\bP_{\hat{\bu}^{\perp}}\bz$.
    Then, $\mathbb{E}\left(\left\Vert\bJ\bz_{\perp}\right\Vert_2^2\right)=\mathrm{tr}\left(\bJ\bP_{\hat{\bu}^{\perp}}\bJ^\top\right)=\left\Vert\bJ\bP_{\hat{\bu}^{\top}}\right\Vert_F^2$.
\end{lemma}

In \cref{alg:input_approximation_orthogonal}, we present a procedure to generate an estimate for $\left\Vert\bJ_{\bx}\bP_{\hat{\bx}^{\perp}}\right\Vert_F^2$ using the estimator of \cref{lem:hutchinson_projected}.
Interestingly, this provides another interpretation of $\mathcal{R}_{\perp}(\bx)$.
Just as we reinterpreted GrokAlign as removing orthogonal components in the Jacobians, we can similarly reinterpret optimizing directly for alignment as a more effective strategy for removing the orthogonal components in the Jacobians.
In particular, note in \cref{eq:orthogonal_contribution} that $\tilde{\bJ}_{\bx}$ has components in the direction of $\bx$.
Since we do not want to include its contribution in our regularization we should instead choose $\boldsymbol{\epsilon}$ orthogonal to $\bx$, in order to maximize the contribution of the orthogonal components of the Jacobian.
This is precisely the procedure outlined in \cref{alg:input_approximation_orthogonal} that generates an estimate of $\left\Vert\bJ_{\bx}\bP_{\hat{\bx}^{\perp}}\right\Vert_F^2$ for $\mathcal{R}_{\perp}(\bx)$.

\begin{algorithm}[ht]
\caption{Estimating Frobenius norms of Jacobians using input random but orthogonal vectors.}\label{alg:input_approximation_orthogonal}
\begin{algorithmic}
\Require Classifier $f$, input $\bx\in\R^d$, Number of projections $K$.
\State $\hat{\bx}\leftarrow\frac{\bx}{\Vert\bx\Vert_2}$
\For{$k=1,\dots,K$}
\State Sample $\bu$ from $\mathcal{N}\left(\boldsymbol{0},\bI_d\right)$.
\State $\bu\leftarrow\bu-\left(\hat{\bx}^\top\bu\right)\hat{\bx}$
\State $\bJ^{(\bu)}_{\bx}\leftarrow\bJ_{\bx}\bu$ \Comment{jvp, requires a forward pass.}
\State $s_k\leftarrow\left\Vert\bJ^{(\bu)}_{\bx}\right\Vert_2^2$
\EndFor
\State $\mathcal{R}_{\perp}(\bx)\leftarrow\frac{1}{K}\sum_{k=1}^Ks_k$\\
\Return $\mathcal{R}_{\perp}(\bx)$
\end{algorithmic}
\end{algorithm}

\subsection{Empirical Comparison}\label{sec:empirical_comparison}

To test these regularization strategies and their hyperparameters, we train a fully connected deep network on MNIST.
The weights of the deep network are scaled by a factor of 2 at initialization, and a weight decay of $0.0001$ is applied.
The hyperparameters we consider are the number of projections used to generate the estimates and the regularization coefficient used to append the regularizer to the loss function; a full list of these can be found in \Cref{tab:hyperparameters}.
We then select the best-performing hyperparameters to compare the regularizers in \Cref{fig:grokalign_ablation}.

As a result of the experiments, we observe that GrokAlign using \Cref{alg:output_approximation} with one projection is the most effective strategy for inducing alignment and robustness in the deep network.
Consequently, we use this implementation as the standard method of GrokAlign.

\section{Experimental Details}\label{sec:experimental_details}

The majority of the experiments were conducted using a combination of NVIDIA TITAN Xs and NVIDIA Quadro RTX 8000s.
Some of the tasks in \Cref{tab:rfms} required greater memory capacity, so an NVIDIA A100-SXM4-80GB was used.
The experiments of \Cref{tab:acceleration} took around 50 GPU hours, whereas the experiment of \Cref{tab:rfms} took only a few hours.

\subsection{\Cref{fig:data_sparsity}}\label{exp:data_sparsity}

For the left panel, we train a fully connected ReLU deep network on a subset of MNIST containing 2000 samples.
The deep network has 4 layers with a hidden width of 256.
It is trained with a batch size of 100, weight decay of $0.0001$, and the AdamW optimizer~\cite{loshchilov2019decoupled} with a learning rate of $0.001$.
GrokAlign is used with strength $0.0001$.
Rotation, translation, and Gaussian noise augmentations are applied during training.
The intensity parameter, say $\gamma\in[0,1]$, controls the strength of these augmentations.
Rotations are applied with maximum angle $30\cdot\gamma$ degrees, translations are applied with maximum shift of $0.2\cdot\gamma$, and Gaussian noise is applied with noise of standard deviation $0.3\cdot\gamma$. 

For the right panel, we use the adversarial training pipeline from \citet{cui2024decoupled}.
More specifically, we train on CIFAR10 using the \texttt{basic} data augmentation strategy and all other default training parameters.\footnote{\url{https://github.com/jiequancui/DKL/tree/main/DKLv2/Adv-training-dkl}}
To form the subsets, we divided the total subset size by the number of classes, and then sample the first images from each class.

\subsection{\Cref{fig:gaussian_log_reg}}\label{exp:gaussian_log_reg}

The dataset comprises of $1000$ samples from Gaussian blobs in ten-dimensional space.
The dataset is split into $80\%$ for training and $20\%$ for testing.

The model uses $50$ center vectors, initialized from the centers obtained by K-means clustering of the data.
The model is trained using full-batch gradient descent with the Adam optimizer with a learning rate of $0.01$ for $1000$ epochs.
weight-decay is applied with a coefficient of $0.001$.
We repeat the experiment five times with random initializations.

\subsection{\Cref{fig:ca_modadd}}\label{exp:ca_modadd}

The training pipeline is identical to that of \citet{nandaProgressMeasuresGrokking2022}.

\subsection{\Cref{tab:acceleration}}\label{exp:acceleration}

This experiment consists of different training setups, and we will detail each one in turn. 
In \Cref{tab:grokking_criteria}, we present the grokked state criteria for each setup.
Throughout each configuration, we maintain the same weight decay strength.
When implementing OrthoGrad, there is no other hyperparameter.
On one seed, we test GrokFast-EMA with the $\alpha\in\{0.8,0.98\}$ and $\lambda\in\{0.1,1.0\}$ values, to set the values for the experiment.
For GrokAlign, we test different $\lambda_{\text{Jac}}$ values less than or equal to $1.0$.

\paragraph{XOR.} 

The setup is similar to that of \citet{xu2025let}, entailing a scalar-output two-layer fully connected network learning on XOR cluster data. 
The XOR cluster data contains $40000$-dimensional vectors of the form $\bx=\left(x_1,x_2,\tilde{\bx}^\top\right)^\top\in\R^{40000}$, where $x_1,x_2\in\left\{\pm1\right\}$ and $\tilde{\bx}\in\R^{39998}$. 
The $400$ samples used to train the network are constructed by sampling entries $x_1$, $x_2$ uniformly from $\left\{\pm1\right\}$ and entries of $\tilde{\bx}$ uniformly from $\left\{\pm\epsilon\right\}$, here we take $\epsilon=0.05$. 
The corresponding label of such a sample is $x_1x_2\in\{\pm1\}$. 
A similar sample of the same size is generated as a test set.

The DN is trained up to $1000$ epochs using full-batch gradient descent with a learning rate of $0.1$ and a weight-decay of $0.1$. To test the adversarial accuracy of the DN, we perturb the last $39998$ components of the test set with random noise of standard deviation $0.2$.

GrokAlign is used with $\lambda_{\text{Jac}}$ equal to $1.0$.
GrokFast is used with $(\alpha,\lambda)=(0.8,0.1)$.

\paragraph{Sparse Parity.} 

This setup is taken from \citet{prietoGrokkingEdgeNumerical2025}. 
It involves performing the binary classification of a bit string based on the parity of the sum of a select few indices.
More specifically, the training distribution consists of $2000$ bit strings of length $40$, with labels equal to the parity of the sum of the first three bits.
We train a DN with two hidden layers of width $200$ on half of this training distribution and test it on the other half.
The DN is trained using the cross-entropy loss function with the AdamW optimizer and a learning rate of $0.01$.
We consider up to $20,000$ epochs.
Weight-decay is applied at $0.1$.
GrokAlign is used with $\lambda_{\text{Jac}}$ equal to $0.1$.
GrokFast is used with $(\alpha,\lambda)=(0.8,0.1)$.

\paragraph{MNIST.} 

Here we adopt a setup similar to that of \citet{liuOmnigrokGrokkingAlgorithmic2022}. 
That is, we train a three-hidden-layer DN on a random $1024$-sample subset of the MNIST classification task. 
The deep network has a constant width of $196$, no bias terms, and its weights are multiplied by 4 at initialization. 
The DN is trained with the AdamW optimizer at a learning rate of $0.001$, a batch size of $128$, and a weight decay of $0.01$.
We consider up to $20,000$ epochs.
GrokAlign is used with $\lambda_{\text{Jac}}$ equal to $0.01$.
GrokFast is used with $(\alpha,\lambda)=(0.8,0.1)$.

\paragraph{Modular Addition.}

This setup is similar to one from \citet{mallinarEmergenceNonneuralModels2025} and involves a one-hidden-layer fully connected DN learning addition modular $61$.
The DN has a width of $256$ and uses a quadratic activation function.
The DN is trained using the AdamW optimizer with a learning rate of $0.001$ and a batch size of $32$.
We consider up to $1,000$ epochs.
A weight-decay of $1.0$.
GrokAlign is used with $\lambda_{\text{Jac}}$ equal to $0.01$.
GrokFast is used with $(\alpha,\lambda)=(0.8,0.1)$.

\subsection{\Cref{tab:rfms}}\label{exp:rfms}

Utilizing the training pipeline of the GitHub repository accompanying \citet{radhakrishnanMechanismFeatureLearning2024},\footnote{\url{https://github.com/aradha/recursive_feature_machines/tree/pip_install/tabular_benchmark_experiments}}, we evaluate the performance of the RFAMs and RFMs on the datasets of \citet{fernandez-delgadoWeNeedHundreds2014}.

\begin{table}[ht]
    \centering
    \caption{
    For each dataset considered in \Cref{tab:rfms}, we consider the relative change of performance metrics induced by using RFAMs instead of RFMs.
    We report the mean and standard deviation of these changes, along with the $p$-value from a one-sided t-test comparing RFAM performance to the RFM baseline.
    }
    \vspace{0.5em}
    \begin{tabular}{cccc}
        \toprule
        Metric & Test Accuracy & Normal Alignment & Attack Success Rate \\
        \midrule
        Mean & $-0.021$& $0.42$ & $-3.26$ \\
        Standard Deviation & $0.052$ & $1.1$ & $14.2$ \\
        $p$-value & $1.0$ & $<0.001$ & $0.007$ \\
        \bottomrule
    \end{tabular}
    \label{tab:rfms_tests}
\end{table}

\section{Mathematical Derivations}\label{sec:derivations}

\subsection{Gaussian Kernel Logistic Regression}

\paragraph{Proof of \Cref{lem:gaussian_jacobian}.}\label{proof:gaussian_jacobian}

Clearly, 
\begin{align*}
    \nabla_{\bx}f_c(\bx)&=\nabla_{\bx}\left(\sum_{i=1}^K\bW_{c,i}\phi_i(\bx)\right)\\&=\sum_{i=1}^K\bW_{c,i}\left(-2\gamma\right)\left(\bx-\boldsymbol{\tau}_i\right)\phi_i(\bx).
\end{align*}
\qed

\subsection{Optimizing the Training Objective with Normal Aligned Classifiers}

\paragraph{Proof of \Cref{thm:optimal_jacobian}.}\label{proof:optimal_jacobian}

We can account for the offset term by adding an extract dimension to our input space.
Namely, let $\tilde{\bx}_i=\left(\bx_i,1\right)$ and $\tilde{\bJ}_{\bx_i}=\left(\bJ_{\bx_i},\bb_{\bx_i}\right)$.
Without loss of generality, we will henceforth let $\tilde{\bx}_i=\bx_i$, $\tilde{\bJ}_{\bx_i}=\bJ_{\bx_i}$ and set $\bb_{\bx_i}=\boldsymbol{0}$.

Since $\ell_i=\ell\left(f\left(\bx_i\right)\right)=\ell\left(\bJ_{\bx_i}\bx_i\right)$, it follows from the chain rule that $$\frac{\partial\ell_i}{\partial\bJ_{\bx_i}^{(c)}}=\frac{\partial\ell}{\partial\bJ_{\bx_i}^{(c)}}\cdot\bx_i.$$For simplicity we suppose that $\ell$ treats all wrong classes equally such that,
\begin{equation}\label{eq:loss_function_condition}
        \frac{\partial\ell_i}{\partial\bJ^{(c)}_{\bx_i}}=\begin{cases}\beta_1(c)\bx_i&c=y_i\\\beta_2(c)\bx_i&c\neq y_i,\end{cases}    
\end{equation}
however, the proof proceeds without this assumption too.

Note that the optimization problem is convex on a convex set, and thus it is sufficient to consider the Karush-Kuhn-Tucker conditions with a Lagrange multiplier. More specifically, since the Frobenius norm constraint implies that $\sum_{c=1}^C\left\Vert\bJ^{(c)}_{\bx_i}\right\Vert_2^2\leq\alpha$, we can consider $$\ell_i^{(\text{KKT})}=\ell_i+\lambda\left(\sum_{c=1}^C\left\Vert\bJ^{(c)}_{\bx_i}\right\Vert_2^2-\alpha\right).$$Thus, 
\begin{equation}\label{eq:kkt_1}
    \boldsymbol{0}=\frac{\partial\ell_i^{(\text{KKT})}}{\partial\bJ^{(c)}_{\bx_i}}=\frac{\partial\ell_i}{\partial\bJ^{(c)}_{\bx_i}}+2\lambda\bJ^{(c)}_{\bx_i}
\end{equation}
for $c=1,\dots,C$, and 
\begin{equation}\label{eq:kkt_2}
    0=\frac{\partial\ell^{(\text{KKT})}}{\partial\lambda}=\sum_{c=1}^C\left\Vert\bJ^{(c)}_{\bx_i}\right\Vert_2^2-\alpha.
\end{equation}
From \cref{eq:kkt_1}, we have  $$0=\sum_{c=1}^C\left\langle\bJ^{(c)}_{\bx_i},\frac{\partial\ell_i^{(\text{KKT})}}{\partial\bJ^{(c)}_{\bx_i}}\right\rangle=\sum_{c=1}^C\left\langle\bJ^{(c)}_{\bx_i},\frac{\partial\ell_i}{\partial\bJ^{(c)}_{\bx_i}}\right\rangle+2\lambda\left\Vert\bJ^{(c)}_{\bx_i}\right\Vert_2^2\overset{\cref{eq:kkt_2}}{=}2\lambda\alpha+\sum_{c=1}^C\left\langle\bJ^{(c)}_{\bx_i},\frac{\partial\ell_i}{\partial\bJ^{(c)}_{\bx_i}}\right\rangle,$$which implies that $\lambda=-\frac{1}{2\alpha}\sum_{c=1}^C\left\langle\bJ^{(c)}_{\bx_i},\frac{\partial\ell_i}{\partial\bJ^{(c)}_{\bx_i}}\right\rangle$. Using this back in \cref{eq:kkt_1}, we deduce that
\begin{equation}\label{eq:kkt_3}
    \boldsymbol{0}=\frac{\partial\ell_i}{\partial\bJ^{(c)}_{\bx_i}}-\frac{1}{\alpha}\sum_{c^\prime=1}^C\left\langle\bJ^{(c^\prime)}_{\bx_i},\frac{\partial\ell_i}{\partial\bJ^{(c^\prime)}_{\bx_i}}\right\rangle\bJ^{(c)}_{\bx_i}
\end{equation}
for $c=1,\dots,C$. Consider the ansatz 
\begin{equation}\label{eq:ansatz}
    \bJ^{(c)}_{\bx_i}=\begin{cases}a_1\bx_i&c=y_i\\a_2\bx_i&c\neq y_i.\end{cases}
\end{equation}
Then for $c=y_i$, \cref{eq:kkt_3} becomes 
\begin{align*}
    \boldsymbol{0}&=\beta_1\bx_i-\frac{1}{\alpha}\left((C-1)a_2\beta_2\left\Vert\bx_i\right\Vert_2^2+a_1\beta_1\left\Vert\bx_i\right\Vert_i^2\right)a_1\bx_i\\&=\left(\alpha\beta_1-\left((C-1)a_2\beta_2+a_1\beta_1\right)a_1\left\Vert\bx_i\right\Vert_2^2\right)\bx_i,
\end{align*}
which implies that
\begin{equation}\label{eq:ansatz_implication_1}
    0=\alpha\beta_1-\left((C-1)a_1a_2\beta_2+a_1^2\beta_1\right)\left\Vert\bx_i\right\Vert_2^2.
\end{equation}
Similarly, when $c\neq y_i$, from \cref{eq:kkt_3} we deduce that
\begin{equation}\label{eq:ansatz_implication_2}
    0=\alpha\beta_2-\left((C-1)a_2^2\beta_2+a_1a_2\beta_1\right)\left\Vert\bx_i\right\Vert_2^2.
\end{equation}
Furthermore, from \cref{eq:kkt_2} we get 
\begin{equation}\label{eq:ansatz_implication_3}
    \alpha=\left\Vert\bx_i\right\Vert_2^2\left(a_1^2+(C-1)a_2^2\right).
\end{equation}
Provided $\beta_1$ and $\beta_2$ are non-zero,
\begin{equation}\label{eq:non_zero_solution}
    \begin{cases}a_1=\frac{\beta_1\sqrt{\alpha}}{\left\Vert\bx_i\right\Vert_2\sqrt{\beta_1^2+(C-1)\beta_2^2}}\\a_2=-\frac{\beta_2\sqrt{\alpha}}{\left\Vert\bx_i\right\Vert_2^2\sqrt{\beta_1^2+(C-1)\beta_2^2}}\end{cases}
\end{equation}
demonstrates that the systems of \cref{eq:ansatz_implication_1,eq:ansatz_implication_2,eq:ansatz_implication_3} form a consistent system of equations that admit a unique solution. If $\beta_2$ equals zero, then 
\begin{equation}\label{eq:zero_solution}
    \begin{cases}a_1=\frac{\sqrt{\alpha}}{\left\Vert\bx_i\right\Vert_2}\\a_2=0\end{cases}
\end{equation}
demonstrates that the systems of \cref{eq:ansatz_implication_1,eq:ansatz_implication_2,eq:ansatz_implication_3} form a consistent system of equations that admit a unique solution. Therefore, $\bJ_{\bx_i}$ as constructed in \cref{eq:ansatz} minimizes the constrained optimization. 

In particular, $\bJ_{\bx_i}$ is Jacobian aligned with $\bc=\sum_{c=1}^Ca_1\mathbb{I}_{\left\{c=y_i\right\}}+a_2\mathbb{I}_{\left\{c\neq y_i\right\}}$.
Moreover, we can recover the offset term as $\bb_{\bx_i}=\bc$. \qed

\qed

\subsection{Normal Aligned Classifiers are Robust}

\begin{lemma}\label{lem:distance_to_boundary}
    Let $\bW\in\R^d\setminus\{0\}$ and consider the binary decision boundary $\left\{\boldsymbol{z}:\bW^{\top}\boldsymbol{z}+\alpha=0\right\}$. For a point $\bx$ with $\bW^{\top}\bx+\alpha\neq0$, we have $$\min\left(\left\{\Vert\beps\Vert_2:\bW^{\top}(\bx+\beps)+\alpha=0\right\}\right)=\frac{\vert\bW^{\top}\bx+\alpha\vert}{\Vert\bW\Vert_2},$$
    achieved by $\beps =-\frac{\bW^T\bx+\alpha}{\Vert\bW\Vert_2^2}\bW.$
\end{lemma}

\begin{proposition}
\label{thm:maximal_local_robustness}
    Let $\bx\in\R^d$ and $\max_{c=1,\dots,C}\left(\frac{\left\vert\bb_{\bx}^{(y)}-\bb_{\bx}^{(c)}\right\vert}{\left\Vert\bJ_{\bx}^{(y)}-\bJ_{\bx}^{(c)}\right\Vert_2}\right)\leq\lambda$,
    then $\rho(\bx)\leq\left\Vert\bx\right\Vert_2+\lambda$.
\end{proposition}

\begin{proof}
    Consider any competing class $c\ne y$. The pairwise decision boundary between $y$ and $c$ is given by
    $$\left(\bJ_{\bx}^{(y)}-\bJ_{\bx}^{(c)}\right)^{\top}\boldsymbol{z}+\left(\bb_{\bx}^{(y)}-\bb_{\bx}^{(c)}\right)=0.$$
    Set $\bW_{y,c}:=\bJ_{\bx}^{(y)}-\bJ_{\bx}^{(c)}\in\R^d$, and write $\bW_{y,c}=\alpha\hat{\bx}+\beta\bu$ where $\hat{\bx}=\frac{\bx}{\Vert\bx\Vert_2}$, $\bu\perp\bx$, and $\alpha,\beta\in\R$. 
    Then, using \Cref{lem:distance_to_boundary}, the minimal perturbation that moves $\bx$ onto that boundary is 
    $$\frac{\left\vert\bW_{y,c}^{\top}\bx+\left(\bb_{\bx}^{(y)}-\bb_{\bx}^{(c)}\right)\right\vert}{\Vert\bW_{y,c}\Vert_2}=\frac{\left\vert\vert\alpha\vert\Vert\bx\Vert_2+\left(\bb_{\bx}^{(y)}-\bb_{\bx}^{(c)}\right)\right\vert}{\sqrt{\alpha^2+\beta^2}}\le\Vert\bx\Vert_2+\lambda.$$
    For fixed $\bJ$, the adversary will pick the minimizing $c$ (the nearest boundary). 
    Therefore, 
    \begin{equation}\label{eq:local_robustness_bound}
        \rho(\bJ;\bx)=\min_{c\neq y}\frac{\left\vert\bW_{y,c}^{\top}\bx\right\vert}{\Vert\bW_{y,c}\Vert_2}\le\Vert\bx\Vert_2+\lambda.
    \end{equation}
    This is true for every admissible $\bJ$. 
    Taking the maximum yields $\rho(\bx)\le\Vert\bx\Vert_2+\lambda$.
\end{proof}

\paragraph{Proof of \Cref{thm:optimally_robust}.}\label{proof:optimally_robust}

Equality in \cref{eq:local_robustness_bound} is only achieved when $\beta=0$ and $\bb_{\bx}^{(y)}-\bb_{\bx}^{(c)}=\lambda$ for each $c\neq y$.
Clearly, the Jacobian aligned solutions of \Cref{thm:optimal_jacobian} satisfy these conditions.

\qed

\subsection{The Theory of Centroids and Radii}

\begin{theorem}[\citealt{balestrieroGeometryDeepNetworks2019}]\label{thm:centroids_layer}
    The $l^\text{th}$ layer of a DN partitions its input space according to a power diagram with centroids $$\bmu^{(l)}_{\omega_{\varphi(\bx)}^{(l)}}=\left(\bA^{(l)}_{\omega_{\varphi(\bx)}^{(l)}}\right)^\top\mathbf{1},$$and radii $$r_{\omega_{\bx}^{(l)}}^{(l)}=\left\Vert\bmu^{(l)}_{\omega_{\varphi(\bx)}^{(l)}}\right\Vert_2^2+2\left(\bb_{\omega_{\varphi(\bx)}^{(l)}}\right)^\top\mathbf{1}.$$
\end{theorem}

\begin{theorem}[\citealt{balestrieroGeometryDeepNetworks2019}]\label{thm:centroids_network}
    The continuous piecewise operation of a DN from the input to the output of the $l^\text{th}$ layer partitions its input space according to a power diagram with centroids $$\bmu_{\omega_{\varphi(\bx)}^{(1\leftarrow\ell)}}^{(1\leftarrow l)}=\left(\bA^{(l-1)}_{\omega_{\varphi(\bx)}^{(l-1)}}\cdots\bA^{(1)}_{\omega_{\varphi(\bx)}^{(1)}}\right)^\top\bmu_{\omega_{\varphi(\bx)}^{(l)}}^{(l)}=:\left(\bA^{(1\leftarrow l-1)}_{\omega_{\varphi(\bx)}^{(1\leftarrow l-1)}}\right)^\top\bmu_{\omega_{\varphi(\bx)}^{(l)}}^{(l)}$$and radii$$r_{\omega_{\bx}^{(1\leftarrow l)}}^{(1\leftarrow l)}=\left\Vert\bmu^{(1\leftarrow l)}_{\omega_{\varphi(\bx)}^{(l\leftarrow l)}}\right\Vert_2^2+2\left(\bmu^{(l)}_{\omega_{\varphi(\bx)}^{(l)}}\right)^\top\bb^{(1\leftarrow l-1)}_{\omega_{\varphi(\bx)}^{(1\leftarrow l-1)}}+2\left(\bb^{(l)}_{\omega_{\varphi(\bx)}^{(l)}}\right)^\top\mathbf{1}$$
\end{theorem}

\paragraph{Proof of \Cref{prop:pd_parameters}.}\label{proof:pd_parameters}
Using \Cref{thm:centroids_layer} and \Cref{thm:centroids_network}, it follows that 
\begin{align*}
    r_{\omega_{\varphi(\bx)}^{(1\leftarrow l)}}^{(1\leftarrow l)}&=\left\Vert\bmu^{(1\leftarrow l)}_{\omega_{\varphi(\bx)}^{(1\leftarrow l)}}\right\Vert_2^2+2\left(\bmu^{(l)}_{\omega_{\varphi(\bx)}^{(l)}}\right)^\top\bb^{(1\leftarrow l-1)}_{\omega_{\varphi(\bx)}^{(1\leftarrow l-1)}}+2\left(\bb^{(l)}_{\omega_{\varphi(\bx)}^{(l)}}\right)^\top\mathbf{1}\\&=\left\Vert\bmu^{(1\leftarrow l)}_{\omega_{\varphi(\bx)}^{(1\leftarrow l)}}\right\Vert_2^2+2\left(\left(\bA^{(l)}_{\omega_{\varphi(\bx)}^{(l)}}\right)^\top\mathbf{1}\right)^\top\bb^{(1\leftarrow l-1)}_{\omega_{\varphi(\bx)}^{(1\leftarrow l-1)}}+2\left(\bb^{(l)}_{\omega_{\varphi(\bx)}^{(l)}}\right)^\top\mathbf{1}\\&=\left\Vert\bmu^{(1\leftarrow l)}_{\omega_{\varphi(\bx)}^{(1\leftarrow l)}}\right\Vert_2^2+2\left(\bA^{(l)}_{\omega_{\varphi(\bx)}^{(l)}}\bb^{(1\leftarrow l-1)}_{\omega_{\varphi(\bx)}^{(l)}}+\bb^{(l)}_{\omega_{\varphi(\bx)}^{(l)}}\right)^\top\mathbf{1}\\&=\left\Vert\bmu^{(1\leftarrow l)}_{\omega_{\varphi(\bx)}^{(1\leftarrow l)}}\right\Vert_2^2+2\left(\bb^{(1\leftarrow l)}_{\omega_{\varphi(\bx)}^{(l)}}\right)^\top\mathbf{1}.
\end{align*}
Extending this to the $L^\text{th}$ yields the desired result.

Similarly, using \Cref{thm:centroids_layer} and \Cref{thm:centroids_network}, it follows that 
\begin{align*}
    \bmu_{\omega_{\varphi(\bx)}^{(1\leftarrow l)}}^{(1\leftarrow l)}&=\left(\bA^{(l-1)}_{\omega_{\varphi(\bx)}^{(l-1)}}\cdots\bA^{(1)}_{\omega_{\varphi(\bx)}^{(1)}}\right)^\top\bmu_{\omega_{\varphi(\bx)}^{(l)}}^{(l)}\\&=\left(\bA^{(l-1)}_{\omega_{\varphi(\bx)}^{(l-1)}}\cdots\bA^{(1)}_{\omega_{\varphi(\bx)}^{(1)}}\right)^\top\left(\bA^{(l)}_{\omega_{\varphi(\bx)}^{(l)}}\right)^\top\mathbf{1}\\&=\left(\bA^{(l)}_{\omega_{\varphi(\bx)}^{(l)}}\cdots\bA^{(1)}_{\omega_{\varphi(\bx)}^{(1)}}\right)^\top\mathbf{1}\\&=\left(\bA^{(1\leftarrow l)}_{\omega_{\varphi(\bx)}^{(1\leftarrow l)}}\right)^\top\mathbf{1}.
\end{align*}
Extending this to the $L^\text{th}$ yields the desired result. \qed

\subsection{The Dynamics of Centroid Alignment}

\paragraph{Proof of \Cref{prop:centroid_dynamics_general}.}\label{proof:centroid_dynamics_general}

From \Cref{lem:centroid_two_layer_network}, observe that
\begin{equation*}
    \partial_t\left(\bmu_{\varphi(\bx)}\right)
    =
    \left(
    \partial_t\left(\bW_2\right)\bQ_{\bx}\bW_1
    +
    \bW_2\bQ_{\bx}\partial_t\left(\bW_1\right)
    +
    \bW_2(\partial_t\bQ_{\bx})\bW_1
    \right)^\top\mathbf{1}.
\end{equation*}

Because $\sigma$ is the ReLU nonlinearity, we have $\sigma''(z)=0$ for all $z\neq 0$. 
Consequently, $\partial_t\bQ_{\bx}=0$ almost everywhere along the gradient flow, since $\bQ_{\bx}=\mathrm{diag}(\sigma'(\bW_1\bx))$ only changes when a pre-activation crosses zero. 
Ignoring these measure-zero events, the dynamics simplify to
\begin{equation*}
    \partial_t\left(\bmu_{\varphi(\bx)}\right)
    =
    \left(
    \partial_t\left(\bW_2\right)\bQ_{\bx}\bW_1
    +
    \bW_2\bQ_{\bx}\partial_t\left(\bW_1\right)
    \right)^\top\mathbf{1},
\end{equation*}
where 
\begin{equation*}
    \partial_t\left(\bW^{(j)}\right)=-\eta\nabla_{\bW^{(j)}}\mathcal{L}
\end{equation*}
for $j=1,2$. 
By the chain rule, one can show that 
\begin{equation*}
    \nabla_{\bW_1}\mathcal{L}=-\frac{1}{n}\sum_{i=1}^n\left(\bW_2\bQ_{\bx_i}\right)^\top\bm_{\bx_i}\bx_i^\top
\end{equation*}
and
\begin{equation*}
    \nabla_{\bW_2}\mathcal{L}=-\frac{1}{n}\sum_{i=1}^n\bm_{\bx_i}\sigma\left(\bW_1\bx_i\right)^\top.
\end{equation*}
Substituting these gradients into the dynamics equation yields
\begin{equation*}
    \partial_t\left(\bmu_{\varphi(\bx)}\right)=\frac{\eta}{n}\sum_{i=1}^n\left(\left(\bm_{\bx_i}^\top\bW_2\bQ_{\bx_i}\bQ_{\bx}\left(\bW_2\right)^\top\mathbf{1}\right)\bx_i+\left(\bW_1\right)^\top\bQ_{\bx}\sigma\left(\bW_1\bx_i\right)\bm_{\bx_i}^\top\mathbf{1}\right).
\end{equation*}
To find the dynamics of the projection, we compute $\partial_t\left(\left\langle\bx,\bmu_{\varphi(\bx)}\right\rangle\right)=\bx^\top\partial_t\left(\bmu_{\varphi(\bx)}\right)$. Distributing $\bx^\top$ into the second term gives $\bx^\top\left(\bW_1\right)^\top\bQ_{\bx}$. 

Because $\sigma$ is the ReLU function, it satisfies the identity $z\sigma^\prime(z) = \sigma(z)$. Consequently, $\bx^\top\left(\bW_1\right)^\top\bQ_{\bx} = \left(\bQ_{\bx}\bW_1\bx\right)^\top = \sigma\left(\bW_1\bx\right)^\top$. 

Applying this simplification, the result follows:
\begin{equation*}
    \partial_t\left(\left\langle\bx,\bmu_{\varphi(\bx)}\right\rangle\right)=\frac{\eta}{n}\sum_{i=1}^n\bm_{\bx_i}^\top\left[\left(\bW_2\bQ_{\bx_i}\bQ_{\bx}\left(\bW_2\right)^\top\right)\bx^\top\bx_i + \left(\sigma\left(\bW_1\bx\right)^\top\sigma\left(\bW_1\bx_i\right)\right)\right]\mathbf{1}.
\end{equation*}

\qed

\paragraph{Proof of \Cref{thm:centroid_alignment_feature_learning}.}\label{proof:centroid_alignment_feature_learning}

By assumption, $\partial_t\left(\bm_{\bx_i}\right)=0$.
Note from \Cref{eq:centroid_innerproduct_ntk} that,
\begin{equation*}
    \partial^2_t\left(\left\langle\bx,\bmu_{\varphi(\bx)}\right\rangle\right)=\frac{\eta}{n}\sum_{i=1}^n\bm_{\bx_i}^\top\partial_t\left(\boldsymbol{\Theta}\left(\bx,\bx_i\right)\right)\mathbf{1}+\frac{\eta}{n}\sum_{i=1}^n\partial_t\left(\bm_{\bx_i}\right)^\top\boldsymbol{\Theta}\left(\bx,\bx_i\right)\mathbf{1}
\end{equation*}
for $\bx\in\R^d$.
Thus, if $\partial^2_t\left(\left\langle\bx,\bmu_{\varphi(\bx)}\right\rangle\right)\neq0$ (i.e.,  an increasing rate of change of centroid inner product) then $\partial_t\left(\boldsymbol{\Theta}\left(\bx,\bx_i\right)\right)\neq0$ for some $\bx_i$ as $\bm_{\bx_i}\neq\boldsymbol{0}$. 
Hence, an increasing rate of change in the centroid inner product implies that the deep network is in the feature-learning regime. \qed

\subsection{Recursive Feature Machines}

\paragraph{Proof of \Cref{prop:feature_matrix_linear_combination}.}\label{proof:feature_matrix_linear_combination}

Suppose that for each $i\in\{1,\dots,n\}$ we have $\nabla f(\bx_i)=\sum_{j=1}^nc_j^i\bx_j$ for $c_j^i\in\R$.
Then,
\begin{align*}
    \bM&=\frac{1}{n}\sum_{i=1}^n\nabla f(\bx_i)\left(\nabla f(\bx_i)\right)^\top\\&=\frac{1}{n}\sum_{i=1}^n\left(\sum_{j=1}^nc_j^i\bx_j\right)\left(\sum_{j=1}^nc_j^i\bx_j\right)^\top\\&=\frac{1}{n}\sum_{i,j=1}^n\tilde{c}_{ij}\bx_i\bx_j^\top.
\end{align*}
Thus, 
\begin{equation*}
    f(\bx)=\sum_{i=1}^n\balpha_i\phi_M(\bx,\bx_i)=\sum_{i=1}^n\balpha_i\exp\left(-\gamma(\bx-\bx_i)^\top\bM(\bx-\bx_i)\right),
\end{equation*}
which implies that $\nabla f(\bx_i)=\bM\bv_i$ for some $\bv_i\in\R^d$, meaning $\nabla f(\bx_i)=\sum_{j=1}^n\tilde{c}_j^i\bx_j$ for some $\tilde{c}_j^i\in\R$.
Therefore, as before, it follows that 
\begin{equation*}
    \bM^\prime:=\frac{1}{n}\sum_{i=1}^n\nabla f(\bx_i)\left(\nabla f(\bx_i)\right)^\top=\sum_{i=1}^nc_i^\prime\bx_i\bx_i^\top.
\end{equation*}

\begin{figure}[ht]
    \centering
    \begin{subfigure}[b]{0.23\textwidth}
        \includegraphics[width=\textwidth]{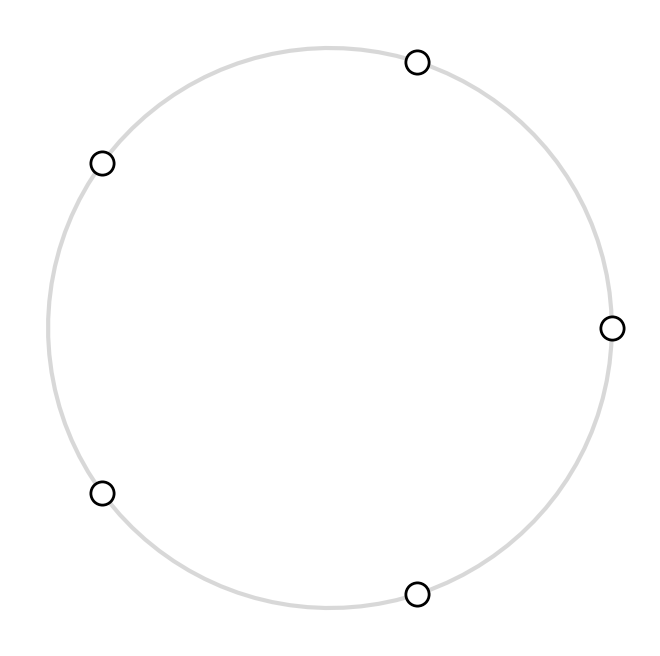}
        \caption*{\scriptsize 5 Sample Training Data}
    \end{subfigure}
    \begin{subfigure}[b]{0.23\textwidth}
        \includegraphics[width=\textwidth]{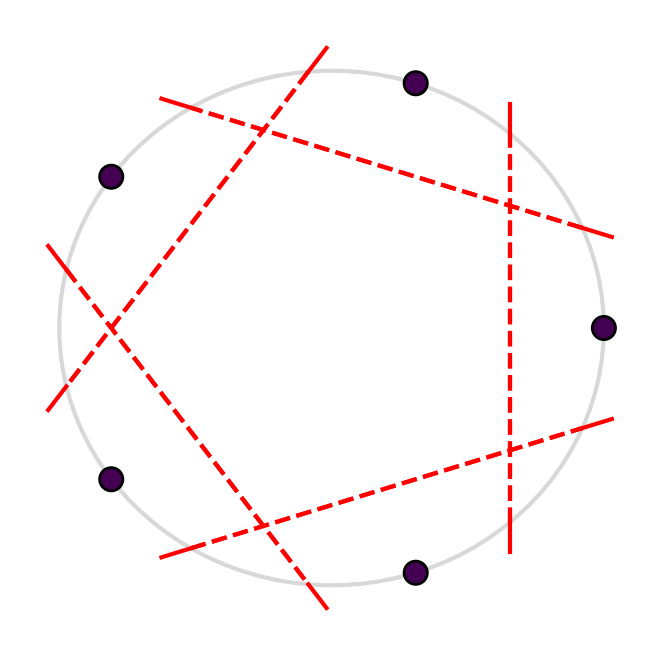}
        \caption*{\scriptsize Normal Aligned Deep Network}
    \end{subfigure}
    \hfill
    \begin{subfigure}[b]{0.23\textwidth}
        \includegraphics[width=\textwidth]{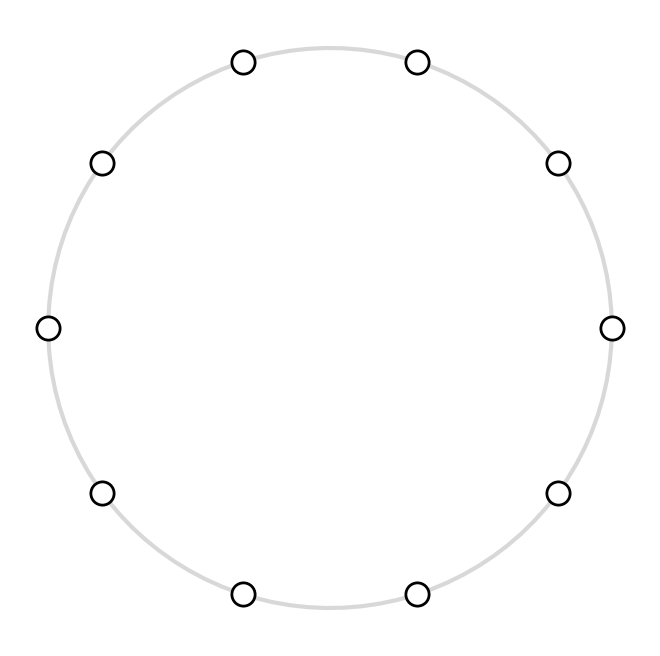}
        \caption*{\scriptsize 10 Sample Training Data}
    \end{subfigure}
    \begin{subfigure}[b]{0.23\textwidth}
        \includegraphics[width=\textwidth]{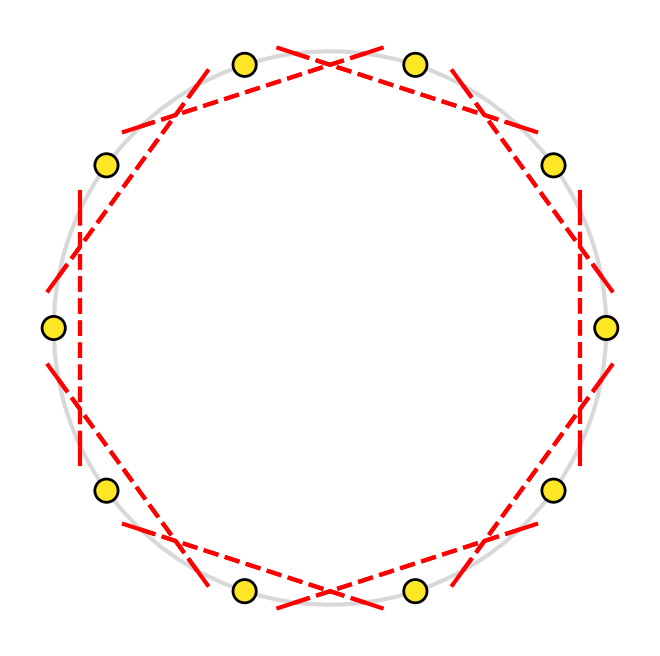}
        \caption*{\scriptsize Normal Aligned Deep Network}
    \end{subfigure}
    \caption{
    \textbf{A deep network with one hidden layer has the capacity to learn a normal-aligned solution for any training set.}
    As the density of the dataset size increases, the irregularity of the deep network — measured by weight norm — increases.
    In the first and second panels (respectively, third and fourth panels), we depict a training set of size 5 (respectively, 10) along with the level sets of the neurons of a one-hidden-layer normal-aligned deep network of width 5 (respectively, 10).
    In the second and fourth panels, the training points are colored (using the ``viridis'' scale) according to the norm of the last-layer weights required for the deep network's output to be 1.
    }
    \label{fig:construction}
\end{figure}

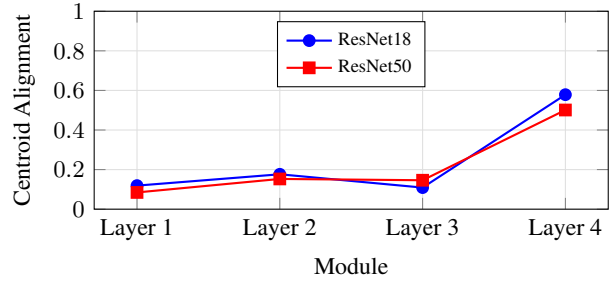
\begin{figure}[ht]
    \centering
    \begin{tikzpicture}
        \begin{axis}[
            width=0.6\textwidth,
            height=4.2cm,
            grid=major,
            grid style={solid, gray!25},
            xtick={0, 1, 2, 3},
            xticklabels={Layer 1, Layer 2, Layer 3, Layer 4},
            legend style={
                at={(0.5,0.95)},
                anchor=north,
                font=\scriptsize
            },
            ylabel={Centroid Alignment},
            xlabel={Module},
            ymin=0, ymax=1,
            tick label style={font=\small},
            label style={font=\small},
        ]
            \addplot[thick, color=blue, mark=*] table [x expr=\coordindex, y=ResNet18, col sep=comma] {data/ResNets.csv};
            \addlegendentry{ResNet18}
            
            \addplot[thick, color=red, mark=square*] table [x expr=\coordindex, y=ResNet50, col sep=comma] {data/ResNets.csv};
            \addlegendentry{ResNet50}
        \end{axis}
    \end{tikzpicture}
    \caption{
    \textbf{Centroid alignment increases for deep layers of deep networks.}
    Here we obtain robust ResNet18 and ResNet50 models trained on CIFAR10~\cite{krizhevskyLearningMultipleLayers2009}, and consider the centroid alignment of the map from the input space of intermediate layers to the output space.
    }
    \label{fig:resnet_alignment}
\end{figure}

\begin{figure}[ht]
    \centering
    \begin{tikzpicture}
    \begin{groupplot}[
        group style={
            group size=3 by 1,
            horizontal sep=1.2cm, 
        },
        width=0.36\textwidth,
        height=4.2cm,
        xlabel={Epochs},
        xmode=log,
        grid=major,
        tick label style={font=\footnotesize},
        label style={font=\footnotesize},
        title style={font=\bfseries\small},
        legend pos=south east, 
        legend style={font=\scriptsize}
    ]
    \nextgroupplot[
        ylabel={Test Accuracy},
        ymin=0, ymax=1.05
    ]
    \foreach \c/\clr in {6/blue, 8/red, 10/green!60!black} {
        \edef\addplotscmd{
            \noexpand\addplot[\clr, thick] table [
                x=epoch, y=test_acc_mean, col sep=comma,
                restrict expr to domain={\noexpand\thisrow{num_classes}}{\c:\c}
            ] {data/gaussian_stats.csv};
            \noexpand\addlegendentry{\c\ Classes}
        }
        \addplotscmd
    }
    \nextgroupplot[
        ylabel={Alignment},
    ]
    \foreach \c/\clr in {6/blue, 8/red, 10/green!60!black} {
        \edef\addplotscmd{
            \noexpand\addplot[\clr, thick] table [
                x=epoch, y=align_mean, col sep=comma,
                restrict expr to domain={\noexpand\thisrow{num_classes}}{\c:\c}
            ] {data/gaussian_stats.csv};
        }
        \addplotscmd
    }
    \nextgroupplot[
        ylabel={Effective Rank},
    ]
    \foreach \c/\clr in {6/blue, 8/red, 10/green!60!black} {
        \edef\addplotscmd{
            \noexpand\addplot[\clr, thick] table [
                x=epoch, y=stab_rank_mean, col sep=comma,
                restrict expr to domain={\noexpand\thisrow{num_classes}}{\c:\c}
            ] {data/gaussian_stats.csv};
        }
        \addplotscmd
    }
    \end{groupplot}
    \end{tikzpicture}
    \caption{
    \textbf{A Gaussian kernel logistic regression model exhibits normal alignment, validating \cref{thm:optimal_jacobian}.}
    Here we train a Gaussian kernel logistic regression model on a ten-dimensional classification problem with either six, eight, or ten classes.
    In the left panel, we monitor the model's test accuracy.
    In the middle panel, we monitor the model's normal alignment.
    In the right panel, we monitor the model's effective rank.
    For further experimental details, see \cref{exp:gaussian_log_reg}.
    }
    \label{fig:gaussian_log_reg}
\end{figure}
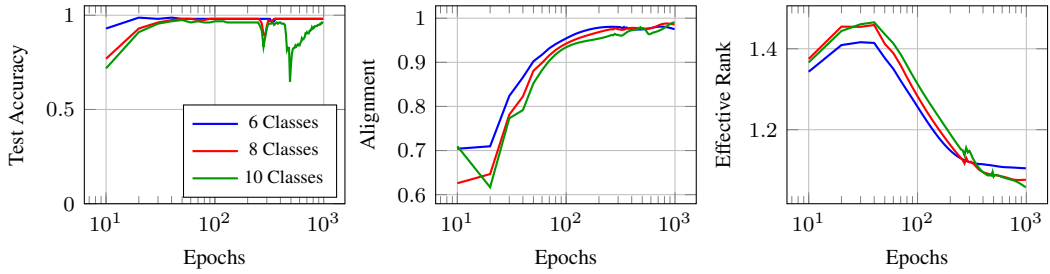

\begin{figure}[ht]
    \centering
    \begin{tikzpicture}
    
    \begin{axis}[
        name=panel1,
        width=5.0cm, height=5.0cm,
        xmode=log,
        xtick={0.1,1,10,100,1000},
        xticklabels={$0$,$1$,$10$,$100$,$1000$},
        xlabel={$\gamma$},
        ylabel={Accuracy (\%)},
        ymin=45, ymax=100,
        ytick={50,60,70,80,90,100},
        tick label style={font=\small},
        label style={font=\small},
        legend style={
            font=\tiny,
            at={(0.03,0.03)},
            anchor=south west,
            draw=gray!60,
            fill=white,
            fill opacity=0.9,
        },
        ]
        \addplot[color=blue, mark=*, mark size=2.2pt, line width=1.2pt] 
            table[x={Gamma},y={Final Clean Accuracy}] {\summarydata};
        \addlegendentry{Clean Accuracy}
         
        \addplot[color=red, mark=square*, mark size=2.0pt,line width=1.2pt]
            table[x={Gamma},y={Final PGD Accuracy}] {\summarydata};
        \addlegendentry{Robust Accuracy}
     
    \end{axis}

    \begin{axis}[
        name=panel2,
        at={(panel1.north east)},
        anchor=north west,
        xshift=1.4cm,
        width=5cm, height=5cm,
        xlabel={Epoch},
        ylabel={Jacobian norm},
        xmin=0, xmax=100,
        ymode=log,
        tick label style={font=\small},
        label style={font=\small},
    ]
     
    \addplot[color=vir1,    line width=1.2pt] table[x=epoch, y=jac_norm] {\histzero};
    \addplot[color=vir2,    line width=1.2pt] table[x=epoch, y=jac_norm] {\histone};
    \addplot[color=vir3,   line width=1.2pt] table[x=epoch, y=jac_norm] {\histten};
    \addplot[color=vir4,  line width=1.2pt] table[x=epoch, y=jac_norm] {\hstone};
    \addplot[color=vir5, line width=1.2pt] table[x=epoch, y=jac_norm] {\histthousand};
     
    \end{axis}
 
    % -------------------------------------------------------
    % PANEL 3 — Bias norm trajectory through training
    % -------------------------------------------------------
    \begin{axis}[
        name=panel3,
        at={(panel2.north east)},
        anchor=north west,
        xshift=1.2cm,
        width=5cm, height=5cm,
        xlabel={Epoch},
        ylabel={Offset norm},
        xmin=0, xmax=100,
        tick label style={font=\small},
        label style={font=\small},
        legend style={
            font=\scriptsize,
            at={(0.97,0.03)},
            anchor=south east,
            draw=gray!60,
            fill=white,
            fill opacity=0.9,
        },
    ]
     
    \addplot[color=vir1,    line width=1.2pt] table[x=epoch, y=bias_norm] {\histzero};
    \addlegendentry{$\gamma=0$}
    \addplot[color=vir2,    line width=1.2pt] table[x=epoch, y=bias_norm] {\histone};
    \addlegendentry{$\gamma=1$}
    \addplot[color=vir3,   line width=1.2pt] table[x=epoch, y=bias_norm] {\histten};
    \addlegendentry{$\gamma=10$}
    \addplot[color=vir4,  line width=1.2pt] table[x=epoch, y=bias_norm] {\hstone};
    \addlegendentry{$\gamma=100$}
    \addplot[color=vir5, line width=1.2pt] table[x=epoch, y=bias_norm] {\histthousand};
    \addlegendentry{$\gamma=1000$}
     
    \end{axis}
 
    \end{tikzpicture}
    \caption{\textbf{Optimal classifiers learn solutions with input-output Jacobians that are non-zero}.
    Here, we train a fully connected deep network on a subset of MNIST with 1000 examples across 100 epochs.
    During training, PGD attacks are applied to the batches, a weight decay of $0.0001$ is used, and a Frobenius norm penalty is applied to the loss function with weight $\gamma$.
    In the first panel, we report the model's accuracy on the test set (`Clean Accuracy') and on PGD-perturbed test samples (`Robust Accuracy') at the end of training.
    In the second and third panels, we record the average Jacobian and offset of the model on the training data, respectively.}
    \label{fig:gamma_analysis}
\end{figure}
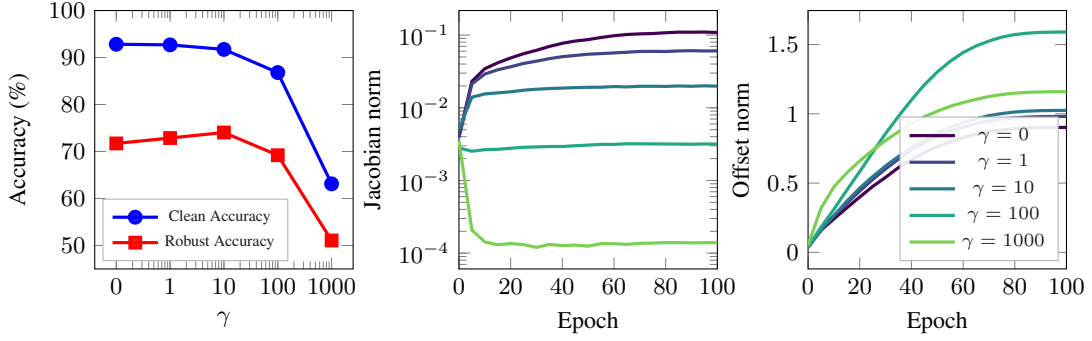

\begin{table}[ht]
    \centering
    \caption{\small Here we state the criteria used to identify the grokked state of a DN for the experiments of \Cref{tab:acceleration}.}
    \label{tab:grokking_criteria}
    \vspace{0.5em}
    \begin{tabular}{cc}
    \toprule
        Setup & Criterion \\
        \midrule
        XOR & Test accuracy and adversarial accuracy greater than $95\%$ \\
        Sparse Parity & Test accuracy greater than $90\%$ \\
        MNIST - Cross Entropy & Test accuracy greater than $80\%$ \\
        MNIST - Squared Error & Test accuracy greater than $80\%$ \\
        Modular Addition & Test accuracy greater than $99\%$ \\
        \bottomrule
    \end{tabular}
\end{table}

\begin{table}[ht]
    \centering
    \caption{
    Here we detail the performance of RFAMs on the tasks of \citet{erickson2025tabarena} for which an $\alpha$ value less than one performed the best.
    }
    \begin{tabular}{lcccc}
         \toprule
         Task Name & Dataset Size & Number of Features & $\alpha$ & Performance Improvement \\
         \midrule
         Bank Customer Churn & $10000$ & $11$ & $0.0$ & $0.52\%$ \\
         Fitness Club & $1500$ & $7$ & $0.0$ & $2.78\%$ \\
         Give Me Some Credit & $150000$ & $11$ & $0.0$ & $11.53\%$ \\
         Website Phishing & $1353$ & $10$ & $0.0$ & $9.52\%$ \\
         Churn & $5000$ & $20$ & $0.001$ & $3.12\%$ \\
         Amazon Employee Access & $32769$ & $10$ & $0.001$ & $4.45\%$ \\
         Bank Marketing & $45211$ & $14$ & $0.1$ & $1.11\%$ \\
         APS Failure & $76000$ & $171$ & $0.1$ & $2.01\%$ \\
         Hiva Agnostic & $3845$ & $1618$ & $0.1$ & $6.20\%$ \\
         NATICUS Droid & $7491$ & $87$ & $0.1$ & $6.03\%$ \\
         SDSS17 & $78053$ & $12$ & $0.1$ & $7.22\%$ \\
         QSAR-TID-11 & $5742$ & $1025$ & $0.1$ & $0.97\%$ \\
         Seismic Bumps & $2584$ & $16$ & $0.1$ & $1.93\%$ \\
         Splice & $3190$ & $61$ & $0.1$ & $10.09\%$ \\
         Wine Quality & $6497$ & $13$ & $0.1$ & $0.04\%$ \\
         MIC & $1699$ & $13$ & $0.1$ & $5.80\%$ \\
         KddCup09 Appetency & $50000$ & $213$ & $0.1$ & $16.19\%$ \\
         \bottomrule
    \end{tabular}
    \label{tab:tabarena}
\end{table}

\begin{table}[ht]
    \centering
    \caption{
    Here we specify which hyper-parameters we considered when testing the different regularization strategies outlined in \Cref{sec:ablation}.
    In \textbf{bold} we identify which hyper-parameters were selected for the best-performance and subsequently considered in \Cref{fig:grokalign_ablation}.
    }
    \begin{tabular}{ccc}
        \toprule
        Regularization Strategy & Number of Projections & Weighting Coefficients \\
        \midrule
        $\mathcal{R}$ (\Cref{alg:input_approximation}) & $\{\mathbf{1},2,4\}$ & $\{0.0001,0.001,0.01,\mathbf{0.1},1.0\}$ \\
        $\mathcal{R}$ (\Cref{alg:output_approximation}) & $\{\mathbf{1},2,4\}$ & $\{0.0001,0.001,\mathbf{0.01},0.1\}$ \\
        $\mathcal{R}_{\perp}$ & $\{\mathbf{1},2,4\}$ & $\{0.0001,0.001,0.01,\mathbf{0.1},1.0\}$ \\
        $\mathcal{R}_{\text{Nuc}}$ & $\{\mathbf{1},2,4\}$ & $\{0.0001,0.001,\mathbf{0.01},0.1\}$ \\
        \bottomrule
    \end{tabular}
    \label{tab:hyperparameters}
\end{table}

\begin{figure}[ht]
    \centering
    \begin{tikzpicture}
        \begin{groupplot}[
            group style={
                group size=3 by 1,
                horizontal sep=1.2cm,
            },
            width=0.36\textwidth,
            height=4.2cm,
            xlabel={Epochs},
            xmode=log,
            grid=major,
            grid style={dashed, gray!30},
            enlarge x limits=false,
            legend style={
                legend columns=-1,
                at={(2.0,-0.6)},
                anchor=south
            },
            tick label style={font=\footnotesize},
            label style={font=\footnotesize},
            legend style={font=\scriptsize}
        ]

        \nextgroupplot[
            ylabel={Robust Accuracy}
        ]

        \addplot[red, thick] table [
            x=step,
            y=accuracy/adv,
            col sep=comma,
            discard if not={label}{GAi}
        ] {data/grokalign_ablation_stats.csv};
        \addlegendentry{$\mathcal{R}$ (\Cref{alg:input_approximation})}

        \addplot[blue, thick] table [
            x=step,
            y=accuracy/adv,
            col sep=comma,
            discard if not={label}{GA}
        ] {data/grokalign_ablation_stats.csv};
        \addlegendentry{$\mathcal{R}$ (\Cref{alg:output_approximation})}

        \addplot[green!60!black, thick] table [
            x=step,
            y=accuracy/adv,
            col sep=comma,
            discard if not={label}{GAiPerp}
        ] {data/grokalign_ablation_stats.csv};
        \addlegendentry{$\mathcal{R}_{\perp}$}

        \addplot[orange, thick] table [
            x=step,
            y=accuracy/adv,
            col sep=comma,
            discard if not={label}{NN}
        ] {data/grokalign_ablation_stats.csv};
        \addlegendentry{$\mathcal{R}_{\text{Nuc}}$}

        \addplot[purple, thick] table [
            x=step,
            y=accuracy/adv,
            col sep=comma,
            discard if not={label}{WD}
        ] {data/grokalign_ablation_stats.csv};
        \addlegendentry{Baseline}

        \nextgroupplot[
            ylabel={Normal Alignment}
        ]

        \addplot[blue, thick] table [x=step, y=jacobian_alignment, col sep=comma, discard if not={label}{GA}] {data/grokalign_ablation_stats.csv};
        \addplot[red, thick] table [x=step, y=jacobian_alignment, col sep=comma, discard if not={label}{GAi}] {data/grokalign_ablation_stats.csv};
        \addplot[green!60!black, thick] table [x=step, y=jacobian_alignment, col sep=comma, discard if not={label}{GAiPerp}] {data/grokalign_ablation_stats.csv};
        \addplot[orange, thick] table [x=step, y=jacobian_alignment, col sep=comma, discard if not={label}{NN}] {data/grokalign_ablation_stats.csv};
        \addplot[purple, thick] table [x=step, y=jacobian_alignment, col sep=comma, discard if not={label}{WD}] {data/grokalign_ablation_stats.csv};

        \nextgroupplot[
            ylabel={Effective Rank}
        ]

        \addplot[blue, thick] table [x=step, y=effective_local_rank, col sep=comma, discard if not={label}{GA}] {data/grokalign_ablation_stats.csv};
        \addplot[red, thick] table [x=step, y=effective_local_rank, col sep=comma, discard if not={label}{GAi}] {data/grokalign_ablation_stats.csv};
        \addplot[green!60!black, thick] table [x=step, y=effective_local_rank, col sep=comma, discard if not={label}{GAiPerp}] {data/grokalign_ablation_stats.csv};
        \addplot[orange, thick] table [x=step, y=effective_local_rank, col sep=comma, discard if not={label}{NN}] {data/grokalign_ablation_stats.csv};
        \addplot[purple, thick] table [x=step, y=effective_local_rank, col sep=comma, discard if not={label}{WD}] {data/grokalign_ablation_stats.csv};

        \end{groupplot}
    \end{tikzpicture}

    \caption{
    \textbf{GrokAlign is the most effective regularization strategy for inducing normal alignment in deep networks.}
    We compare the regularization strategies of \Cref{tab:hyperparameters}, along with a baseline that utilizes no regularization, at inducing normal alignment and reducing the effective ranks of a DN's Jacobians.
    The deep network is trained on MNIST.
    }
    \label{fig:grokalign_ablation}
\end{figure}
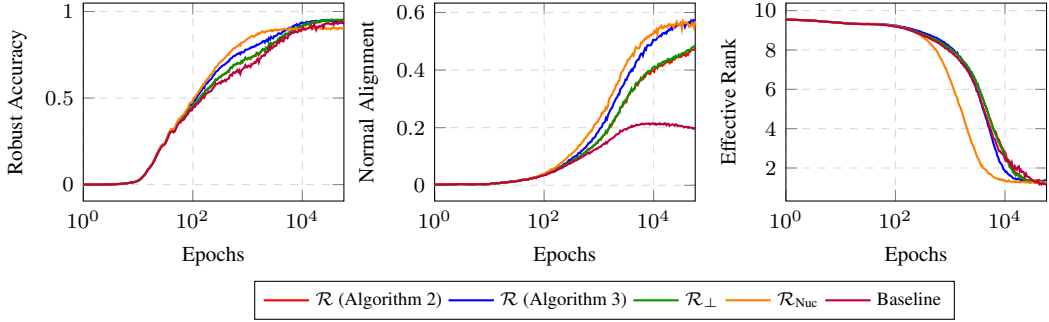

\end{document}